\documentclass{article}
\usepackage[preprint]{neurips_2025}
\workshoptitle{Accepted at the PPML Workshop at EurIPS 2025 and WiML Workshop at NeurIPS 2025}

\usepackage[utf8]{inputenc} % allow utf-8 input
\usepackage[T1]{fontenc}    % use 8-bit T1 fonts
\usepackage{hyperref}       % hyperlinks
\usepackage{url}            % simple URL typesetting
\usepackage{booktabs}       % professional-quality tables
\usepackage{amsfonts}       % blackboard math symbols
\usepackage{nicefrac}       % compact symbols for 1/2, etc.
\usepackage{microtype}      % microtypography
\usepackage{xcolor}         % colors

\usepackage{comment}
\usepackage{multirow}

\newcommand{\eqdef}{\mathrel{\mathop:}=}
\newcommand{\change}[1]{\textcolor{black}{#1}}
\newcommand{\squeeze}{\textstyle}
\usepackage{amsthm, amsmath, amssymb}
\usepackage{thmtools}\usepackage{thmtools, thm-restate}

\newtheorem{assumption}{Assumption}
\newtheorem*{assumption*}{Assumptions}
\newtheorem{Lemma}{Lemma}

\usepackage{algorithm}
\usepackage[noend]{algpseudocode}
\usepackage{graphicx} 
\usepackage{wrapfig}
\usepackage{subcaption}
\usepackage{enumitem}

\definecolor{myblue}{RGB}{31,119,180} %C0

\newcommand{\R}{\mathbb{R}}
\newcommand{\E}{\mathbb{E}}
\newcommand{\q}{q_{\omega}}
\newcommand{\squeezevspace}{\vspace{-0.5em}}
\title{DP-MicroAdam: Private and Frugal Algorithm for Training and Fine-tuning}
\author{
Mihaela Hudi\c{s}teanu \quad Nikita P. Kalinin \quad Edwige Cyffers \vspace{2.75pt} \\
Institute of Science and Technology Austria (ISTA)
}

\begin{document}
\maketitle

\begin{abstract}
Adaptive optimizers are the de facto standard in non-private training as they often enable faster convergence and improved performance. In contrast, differentially private (DP) training is still predominantly performed with DP-SGD, typically requiring extensive compute and hyperparameter tuning.
We propose DP-MicroAdam, a memory-efficient and sparsity-aware adaptive DP optimizer. We prove that DP-MicroAdam converges in stochastic non-convex optimization at the optimal $\mathcal{O}(1/\sqrt{T})$ rate, up to privacy-dependent constants. 
Empirically, DP-MicroAdam outperforms existing adaptive DP optimizers and achieves competitive or superior accuracy compared to DP-SGD across a range of benchmarks, including CIFAR-10, large-scale ImageNet training, and private fine-tuning of pretrained transformers. 
These results demonstrate that adaptive optimization can improve both performance and stability under differential privacy.
\end{abstract}

\section{Introduction}
Machine learning models are often trained on sensitive data, raising concerns that they may reveal information about individual training examples. Differential Privacy (DP)~\citep{dwork2006calibrating} has become the standard framework to address these risks, providing mathematical guarantees that the output of an algorithm does not depend significantly on any single data point.
The most widely used method for private training is DP-SGD~\citep{abadi2016}, which enforces privacy through per-sample gradient clipping, addition of calibrated Gaussian noise, and early stopping. However, DP-SGD has significant limitations: it is sensitive to hyperparameters and can converge slowly in practice. In non-private optimization, SGD is often not competitive in comparison to adaptive optimizers, such as Adam~\citep{kingma2017adam}, %whose name derives from Adaptive Moment estimation,
which computes individual adaptive learning rates for
different coordinates based on estimates of first and second moments of the gradients. Using adaptive optimizers in DP is thus a natural idea.

Unfortunately, in the differentially private regime, the combination of clipping and noise introduces bias into the moment estimates, leading to degraded performance. Recent work~\citep{ganesh2025} has shown that careful de-biasing strategies can improve the stability of DP-Adam, but with limited gains in overall accuracy. Another line of research has explored sparsity for private fine-tuning: methods such as DP-BiTFiT~\citep{bu2024dpbitfit} and SPARTA~\citep{jang2025sparta} show that updating only a small subset of parameters during fine-tuning can reduce the impact of noise. 

In this work, we propose DP-MicroAdam, a differentially private optimizer that combines standard DP mechanisms with the update strategy of MicroAdam~\citep{modoranu2024microadam}: top-$k$ gradient selection, error feedback, quantization, and ring-buffer moment estimation. Our main contributions are as follows.  
(i) We analyze the convergence of DP-MicroAdam for stochastic non-convex optimization and show that it matches the optimal non-private convergence rate of $\mathcal{O}(1/\sqrt{T})$, up to a constant factor depending only on the gradient clipping threshold and noise variance. 
(ii) We conduct a comprehensive empirical evaluation covering three scenarios: CIFAR-10 training, large-scale ImageNet training, and private fine-tuning of pretrained vision transformers on CIFAR-10 and CIFAR-100. 
(iii) We show that DP-MicroAdam consistently outperforms previous private adaptive optimizers such as DP-Adam, DP-AdamBC, and Scale-then-Privatize, and remains competitive with DP-SGD while exhibiting reduced sensitivity to clipping and learning-rate choices and lowering memory usage.
\section{Background}

\subsection{Adam in the Non-Private Setting}
In this section, we recall how Adam~\citep{kingma2017adam} minimizes a differentiable objective $f: \mathbb{R}^d \to \mathbb{R}$ with respect to parameters $\theta$. 
At each iteration $t$, Adam maintains exponential moving averages of the first and second moments of the stochastic gradient $g_t = \nabla f(\theta_t) \in \mathbb{R}^d$:
\[
m_t = \beta_1 m_{t-1} + (1-\beta_1) g_t, 
\qquad 
v_t = \beta_2 v_{t-1} + (1-\beta_2) g_t^2,
\]
where all operations are coordinate-wise ($g_t^2 = g_t \odot g_t$) and $\beta_1, \beta_2 \in [0,1)$ are decay parameters.
With $m_0 = v_0 = 0$, both estimators are initially biased toward zero. 
Indeed, unrolling the recursions and taking expectations under unbiased stochastic gradients $\mathbb{E}[g_i] = \nabla f(\theta_i)$ gives:
\[
\mathbb{E}[m_t] = (1 - \beta_1) \sum_{i=1}^t \beta_1^{t-i} \nabla f(\theta_i) \approx (1 - \beta_1^t)\mathbb{E}[g_t],
\quad
\mathbb{E}[v_t] = (1 - \beta_2) \sum_{i=1}^t \beta_2^{t-i} \mathbb{E}[g_i^{2}] \approx (1 - \beta_2^t)\mathbb{E}[g_t^{2}],
\]
where the approximation assumes slowly varying gradients, i.e.\ $\nabla f(\theta_i) \approx \nabla f(\theta_t)$.  
This motivates the bias-corrected forms $\hat{m}_t = m_t/(1 - \beta_1^t)$ and $\hat{v}_t = v_t/(1 - \beta_2^t)$, leading to the update:
\[
\theta_{t+1} = \theta_t - \eta \frac{\hat{m}_t}{\sqrt{\hat{v}_t}+\epsilon_s},
\]
where $\eta$ is the learning rate and $\epsilon_s$ is a small constant for numerical stability.
Adam’s success stems from three properties.  
First, it scales updates adaptively: each coordinate’s step is normalized by its historical gradient magnitude, yielding an effective per-coordinate learning rate $\eta_{t,j} = \eta / (\sqrt{\hat{v}_{t,j}} + \epsilon_s)$.
Second, momentum through $m_t$ reduces stochastic noise and speeds up convergence, especially on ill-conditioned objectives. 
Third, Adam is robust to hyperparameter choices: the default settings ($\beta_1{=}0.9$, $\beta_2{=}0.999$, $\eta{=}10^{-3}$, $\epsilon_s{=}10^{-8}$) perform reliably across diverse models and datasets.

Adam extends AdaGrad~\citep{duchi2011adagrad} by using exponential moving averages to temper AdaGrad’s rapid learning-rate decay and by adding momentum for smoother updates. However, the interaction between these components complicates convergence analysis. Early work showed that Adam may diverge even on simple convex problems~\citep{reddi2019convergenceadam}.  
Despite this, Adam remains stable in practice and is the default optimizer in deep learning, ranking among the most cited works in machine learning.

\subsection{Challenges for Differentially Private Adam and de-biasing methods}
In the differentially private setting, each stochastic gradient $g_t$ is replaced by a privatized version
\[
\tilde{g}_t = \frac{1}{B}\!\sum_{i=1}^{B}\!\mathrm{clip}(\nabla_\theta f(\theta_t, d_i), C) 
+ \zeta_t \text{ where } \zeta_t \sim \mathcal{N}(0, \sigma^2 C^2 I).
\]
Clipping is a nonlinear transformation that introduces bias whenever the true gradient norm exceeds the threshold $C$, while the added noise inflates the variance of all coordinates. These operations distort the statistics on which Adam relies: the second-moment estimate $v_t$ is overestimated by roughly $\sigma^2 C^2$, and the first-moment estimate $m_t$ accumulates temporally correlated noise. Overestimating $v_t$ decreases the effective learning rate too aggressively for low-variance coordinates. 
The optimizer thus loses geometry-aware scaling, updating all parameters with nearly uniform step sizes. The exponential averaging in $m_t$ and $v_t$ propagates these distortions over time, often leading to slow or unstable convergence. This helps explain why DP-Adam often underperforms DP-SGD in practice and has motivated de-biasing techniques.

\textbf{DP-AdamBC}~\citep{tang2023dpadambc} corrects the second-moment estimate by subtracting the expected contribution of the noise, updating parameters as
\[
\theta_{t+1} \;=\; \theta_{t} \;-\; \eta \cdot \hat{m}_t \;/\;
\sqrt{\,\max\!\left\{ \hat{v}_t- \tfrac{C^2 \sigma^2}{B^2}, \;\epsilon_s^2 \right\}}.
\]
Empirical results reported by~\citet{tang2023dpadambc} show that, on CIFAR-10 with a 5-layer CNN model and under $(7, 10^{-5})$-DP, DP-AdamBC improves over DP-Adam (63.4\% vs.\ 62.2\%) but still remains below DP-SGD, which reaches 65.3\%.  

\textbf{"Scale-then-privatize"} methods~\citep{li2022sideinfo, li2023delayedprecon} first rescale gradients according to prior variance estimates before applying clipping and noise:
\[
s_t = (\sqrt{\hat{v}_{t-1}} + \epsilon_s)^{-1}, \quad
\tilde{g}_{t,i} = \mathrm{clip}(s_t \odot g_{t,i}, C), \quad
\tilde{g}_t = \left(\tfrac{1}{B}\!\sum_{i \in B}\tilde{g}_{t,i} + \zeta_t\right) \!/\! s_t.
\]
Experimental results from~\citep{li2022sideinfo} show that on MNIST under $(1.25,  10^{-5})$-DP, AdaDPS ("Scale-then-privatize") achieves 95.4\% accuracy beating DP-Adam (93.3\%) and DP-SGD (92.7\%).  

While these approaches are encouraging and partially restore Adam’s adaptivity, their effectiveness is limited in high-dimensional settings or when the signal-to-noise ratio is low.

\subsection{MicroAdam (Non-Private)}
MicroAdam~\citep{modoranu2024microadam} is a recent non-private optimizer that reduces the memory and communication overhead of Adam while preserving the theoretical convergence guarantees.
While standard Adam uses dense first and second-moment estimates for all parameters, MicroAdam decreases memory cost by storing and updating only a sparse representation of past gradients, implemented through four key mechanisms:
\begin{itemize}[leftmargin=1em]
    \item \textbf{Top-$k$ selection:} At each iteration, only the largest $k$ gradient coordinates (around 1\% of all parameters) are used for the updates.
    \item \textbf{Error feedback:} The coordinates not selected by the top-$k$ operator are added to a buffer, ensuring that information is not lost and mitigating the bias introduced by sparsification.
    \item \textbf{Quantization:} The accumulator is quantized to a small number of bits (e.g., 4), decreasing storage and communication costs while preserving numerical stability.
    \item \textbf{Sliding-window moment estimation:} Instead of storing full-history moment vectors, which would become dense and negate the benefits of sparsification, MicroAdam reconstructs them from a fixed-size buffer of recent sparse gradients, achieving accurate moment tracking with constant memory.
\end{itemize}
\section{The DP-MicroAdam Algorithm}
DP-MicroAdam builds upon MicroAdam by incorporating the standard components of private stochastic optimization. At each step, minibatches are drawn via Poisson subsampling, per-sample gradients are clipped to a fixed $\ell_2$-norm, and Gaussian noise is added to the aggregated gradient. The privatized gradients are then passed through MicroAdam’s sparse and memory-efficient update rule.

This design allows us to avoid the limitations of previous DP adaptive optimizers. 
The top-$k$ operator prioritizes gradient coordinates with higher signal-to-noise ratios, thus reducing the bias introduced by Gaussian noise on smaller variance coordinates. The error-feedback mechanism complements this by accumulating the residuals of clipped gradients. Over time, the residuals of unselected coordinates grow until they exceed the top-$k$ threshold, while noise tends to cancel each other. This process allows DP-MicroAdam to recover useful low-magnitude signals, reducing bias and improving the stability of adaptive private optimization.

As all operations in DP-MicroAdam are applied to the already privatized gradients,  
the \emph{post-processing property} of DP~\citep{dwork2006calibrating}, ensures that it is enough to compute the privacy guarantees of the gradients to derive the privacy budget of the full update trajectory. Algorithms~\ref{alg:dp-microadam} and~\ref{algorithm:procedures} define the complete DP-MicroAdam procedure. The steps highlighted in \textcolor{myblue}{blue} indicate the modifications required to enforce differential privacy, while the remaining parts follow the standard MicroAdam updates. 

\begin{minipage}[ht]{0.523\textwidth}
    \begin{algorithm}[H]
    \caption{DP-MicroAdam} 
    \label{alg:dp-microadam}
    \begin{algorithmic}[1]
        \State Input: $\beta_1, \beta_2, \epsilon_s, \mathcal{G}, T, d, k, \textcolor{myblue}{C}, \textcolor{myblue}{\sigma}$
        \State $m_0, v_0 \leftarrow 0_d, 0_d$
        \Statex $\delta_1, \Delta_1 \leftarrow 0, 0$
        \Statex $e_1 \leftarrow 0_d^{4b}$
        \For{$t=\{1, 2, ..., T\}$}
            \State \textcolor{myblue}{Subsample batch $\{d_1, \dots, d_B\}$}
            \State \textcolor{myblue}{$g_t \leftarrow  \frac{1}{B} \left(\sum_{i=1}^B \mathrm{clip}(\nabla f(\theta_t, d_i), C) + \zeta_{t}\right)$}
            \State $a_t \leftarrow g_t + Q^{-1}(e_t, \delta_t, \Delta_t)$  
            \State $\mathcal{I}_t, \mathcal{V}_t \gets T_k(|a_t|)$  
            \State $\mathcal{G}_{i,:} \gets (\mathcal{I}_t, \mathcal{V}_t)$ 
            \State $a_t[\mathcal{I}_t] \leftarrow 0$  
            \State $\delta_{t+1}, \Delta_{t+1} \leftarrow \min(a_t), \max(a_t)$
            \State $e_{t+1} \leftarrow Q(a_t, \delta_{t+1}, \Delta_{t+1})$ 
            \State $\hat{m}_t \leftarrow \mathrm{ADAMStats}(\beta_1, \mathcal{G})$
            \State $\hat{v}_t \leftarrow \mathrm{ADAMStats}(\beta_2, \mathcal{G}^2)$
            \State $\theta_{t+1} \leftarrow \theta_t - \eta_t \frac{\hat{m}_t}{\epsilon_s + \sqrt{\hat{v}_t}}$ 
            \State $i \leftarrow (i + 1) \bmod m$ 
        \EndFor
    \end{algorithmic}
\end{algorithm}
\end{minipage}
\hfill
\begin{minipage}[ht]{0.455\textwidth}
    \begin{algorithm}[H]
        \caption{Adam Statistics, Quantization and Inverse Quantization}\label{algorithm:procedures}
        \begin{algorithmic}[1]
            \Procedure{AdamStats}{$\beta, \mathcal{G}, t, m, d$}
                \State $z \gets 0_d$
                \For{$i \in \{1, 2, ..., \min(t, m)\}$}
                    \State $r \gets (t-i-1) \% m$
                    \State $z[\mathcal{I}_i] \gets z[\mathcal{I}_i] + \beta^r \mathcal{V}_i$
                \EndFor
                \State \textbf{return } $\frac{\change{(1-\beta)z}}{1-\beta^t}$
            \EndProcedure
        \end{algorithmic}
        \vspace{.62em}
        \begin{algorithmic}[1]
            \Procedure{$Q$}{$x, \delta, \Delta, b=4$} 
                \State $u \gets \frac{\Delta-\delta}{2^b - 1}$
                \State $x_Q \gets \lfloor \frac{x -\delta}{u} + \frac{1}{2} \rfloor$
                \State \textbf{return } $x_Q$
            \EndProcedure
        \end{algorithmic}
        \vspace{.62em}
        \begin{algorithmic}[1]
            \Procedure{$Q^{-1}$}{$x_Q, \delta, \Delta, b$} 
                \State $u \gets \frac{\Delta-\delta}{2^b - 1}$
                \State $x \gets x_Q \cdot u + \delta$
                \State \textbf{return } $x$
            \EndProcedure
        \end{algorithmic}
    \end{algorithm}
\end{minipage}

\section{Convergence Guarantees for DP-MicroAdam}
In this section, we establish convergence of DP-MicroAdam under a set of usual hypotheses, and we refer the reader to Appendix~\ref{app:theo} for detailed definitions.

\begin{assumption}
We assume that the following hypotheses hold:
\begin{enumerate}[leftmargin=1em,nosep]
    \item (Compressors) The gradient compressor $\mathcal T$ (top $K$) is $q$-contractive with $0 \leq q < 1$ and the error compressor $Q$ is unbiased and $\omega$-bounded with $\omega \geq 0$,
    \item (Lower bound and smoothness) The loss function \(f\) is lower bounded by \(f^* \in \mathbb{R}\) and \(L\)-smooth, 
    \item (Unbiased and bounded stochastic gradient) For all iterates \(t \geq 1\), the stochastic gradient \(g_t\) is unbiased and uniformly bounded \(G\),
    \item (Bounded variance) For all iterates \(t \geq 1\), the variance of the stochastic gradient \(g_t\) is uniformly bounded by \(\sigma_g^2\).
\end{enumerate}
\label{assum:main}
\end{assumption}   

Finally, following standard analysis assumptions~\cite{NIPS2017_f337d999}, we assume that only a negligible fraction of gradients reach the clipping threshold. This can be ensured, for example, when the loss is $L$-Lipschitz and the clipping bound is set to $C = L$. Under this assumption, clipping has a limited effect, and the privatization step primarily adds Gaussian noise to otherwise unbiased gradients.
Under these assumptions, we obtain the following result.

\begin{restatable}{theorem}{TheoremNonconvex}\label{TheoremNonconvexO}
    Under Assumptions \ref{assum:main}, for $q_{\omega} \eqdef (1+\omega)q < 1$ and a step-size $\eta = \min\{\frac{\epsilon_s}{4LC_0}, \frac{1}{\sqrt{T}}\}$, DP-MicroAdam (Algorithm \ref{alg:dp-microadam}) with noise parameter $\sigma$ satisfies
    \begin{multline*}
    \squeeze
    \frac{1}{T}\sum_{t=1}^T \mathbb{E}[\|\nabla f(\theta_t)\|^2]
    \leq 2C_0\left(\frac{f(\theta_1)-f^*}{\sqrt{T}}
    +\frac{L (\sigma_g^2 \textcolor{myblue}{+d\sigma^2}+ C_2^2\textcolor{myblue}{(}G \textcolor{myblue}{+\sqrt{d}\sigma)^2})}{\epsilon_s\sqrt{T}} + \frac{(1+C_1)\textcolor{myblue}{(}G \textcolor{myblue}{+\sqrt{d}\sigma)^2\sqrt{d}}}{\textcolor{myblue}{\sqrt{T}}\sqrt\epsilon_s} \right)\\
    \squeeze
    + \mathcal{O}\left(\frac{\textcolor{myblue}{(}G \textcolor{myblue}{+\sqrt{d}\sigma)}^4}{T}\right),
\end{multline*}
where
$ C_0\eqdef \sqrt{\frac{\textcolor{myblue}{2 \cdot}4(1+q_{\omega}^2)^3}{(1-q_{\omega}^2)^2}\textcolor{myblue}{(}G \textcolor{myblue}{+\sqrt{d}\sigma)^2}
+\epsilon_s},\, C_1\eqdef \frac{\beta_1}{1-\beta_1}(1+C_2)+\frac{2q_{\omega}}{1-q_{\omega}^2},\, C_2 \eqdef \omega q (1 + \frac{2q_{\omega}}{1-q_{\omega}^2})$.   
\end{restatable}    
\noindent
The terms highlighted in \textcolor{myblue}{blue} correspond to the terms that distinguish the DP convergence guarantees from the non-private ones. The theorem proves that DP does not break convergence, and do not change significantly its speed, keeping on overall $\mathcal{O}(1/\sqrt{T})$ rate.

\section{Experiments}

We evaluate DP-MicroAdam on image classification benchmarks. All experiments are implemented in PyTorch using Opacus~\citep{opacus}. Privacy guarantees are computed using the R\'enyi Differential Privacy (RDP) accountant~\citep{mironov2017rdp} and reported in the standard $(\varepsilon,\delta)$-DP framework. Code is available at \url{https://github.com/MihaelaHudisteanu/DP-Micro-Adam}. In all experiments, we set the clipping bound to $C=1$, as DP-MicroAdam requires no clipping hyperparameter tuning in practice. We also use Adam's default parameters ($\beta_1=0.9$, $\beta_2=0.999$, $\eta=10^{-3}$, $\epsilon_s{=}10^{-8}$). Additional experimental results and implementation details are given in Appendix~\ref{app:moretables}.

%https://anonymous.4open.science/r/DP-Micro-Adam/README.md

\paragraph{CIFAR-10}
Using Wide-ResNet-16-4, a popular architecture in private training, we compare DP-MicroAdam to DP-SGD and several Adam-based baselines: naive DP-Adam, DP-AdamBC~\cite{tang2023dpadambc}, and Scale-then-Privatize~\cite{li2023delayedprecon}. Table~\ref{tab:results_cifar} reports test accuracy for multiple privacy budgets. DP-MicroAdam consistently outperforms competing DP-Adaptive optimizers and DP-SGD, except for the small $\varepsilon = 2$. These findings persist when varying the number of overall training steps (see Appendix~\ref{app:moretables}).

\begin{table}[ht]
\centering
\caption{Test accuracy (\%) on CIFAR-10 for privacy budgets $(\varepsilon, \delta{=}10^{-5})$.}
\vspace{0.5em}
\label{tab:results_cifar}
\begin{tabular}{l|ccccc}
\toprule
\textbf{$\varepsilon$} & \textbf{DP-Adam} & \textbf{DP-AdamBC} & \textbf{Scale-then-Priv.} & \textbf{DP-MicroAdam} & \textbf{DP-SGD} \\
\midrule
2 & 51.7 & \textbf{53.6} & 49.2 & 50.5 & 52.4 \\
4 & 58.5 & 59.0 & 57.3 & \textbf{62.2} & 62.0 \\
6 & 62.7 & 62.5 & 61.6 & \textbf{67.9} & 67.2 \\
8 & 65.5 & 65.7 & 64.4 & \textbf{71.4} & 71.0 \\
\bottomrule
\end{tabular}
\end{table}

\paragraph{ImageNet}
Significantly larger with $1.2\times10^6$ training samples, ImageNet is a challenging dataset for DP training in terms of compute. We benefit here from two properties of DP-MicroAdam: the absence of sensitive hyperparameters to tune and less computationally expensive updates. We outperform DP-SGD~\cite{de2022unlocking} as reported in Table~\ref{tab:results_imagenet} and we are slightly below the state-of-the-art $39.2\%$ of~\cite{sander2023tanburnscalinglaws}, which in contrast relies on extensive hyperparameter tuning (learning rate, momentum, dampening factor, EMA decay), several augmentation strategies, and larger networks.

\begin{table}[ht]
\centering
\caption{Top-1 accuracy (\%) on ImageNet under $(\varepsilon{=}8, \delta{=}8{\times}10^{-7})$-DP.}
\vspace{0.5em}
\label{tab:results_imagenet}
\begin{tabular}{lcccc}
\toprule
\textbf{Method} & \textbf{Model} & \textbf{Hardware} & \textbf{Training time} & \textbf{Top-1 Accuracy (\%)} \\
\midrule
DP-SGD & NF-ResNet-50 & 32$\times$H100 GPUs & 4d & 32.4 \\
DP-MicroAdam & WRN-16-4 & 8$\times$H100 GPUs & 8.5d & \textbf{38.1} \\
\bottomrule
\end{tabular}
\end{table}

\paragraph{Fine-tuning on CIFAR-10 and CIFAR-100}
We fine-tune DeiT models pretrained non-privately on ImageNet, comparing DP-MicroAdam to sparse private fine-tuning methods DP-BiTFiT~\cite{bu2024dpbitfit} and SPARTA~\cite{jang2025sparta}. On CIFAR-100, DP-MicroAdam consistently outperforms both baselines, achieving $83.6\%$ Top-1 accuracy with DeiT-Base, while SPARTA reaches $80.4\%$ and DP-BiTFiT $78.5\%$.

\begin{table*}[ht]
\centering
\caption{Top-1 accuracy (\%) on CIFAR-10 and CIFAR-100 under $(\varepsilon{=}8, \delta{=}10^{-5})$-DP.}
\vspace{0.5em}
\label{tab:results_finetune}
\begin{tabular}{c ccccc}
\toprule
\textbf{Dataset} & \textbf{Model} & \textbf{DP-MicroAdam} & \textbf{SPARTA} & \textbf{DP-BiTFiT} & \textbf{Non-Private} \\
\midrule
CIFAR-10  & DeiT-Tiny  & 92.51$\pm$(0.17) & \textbf{93.75}$\pm$(0.05) & 92.27$\pm$(0.05) & 97.20 \\
           & DeiT-Small & \textbf{96.16}$\pm$(0.13) & 96.12$\pm$(0.07) & 95.25$\pm$(0.02) & 98.14 \\
           & DeiT-Base  & \textbf{97.14}$\pm$(0.19) & 97.05$\pm$(0.03) & 96.30$\pm$(0.03) & 98.53 \\
\cmidrule(l){1-6}
CIFAR-100 & DeiT-Tiny  & \textbf{71.70}$\pm$(0.37) & 70.54$\pm$(0.14) & 69.82$\pm$(0.11) & 84.86 \\
           & DeiT-Small & \textbf{78.80}$\pm$(0.20) & 76.72$\pm$(0.17) & 75.37$\pm$(0.11) & 88.15 \\
           & DeiT-Base  & \textbf{83.63}$\pm$(0.23)& 80.40$\pm$(0.10) & 78.51$\pm$(0.08) & 90.53 \\
\bottomrule
\end{tabular}
\end{table*}

\section{Conclusion}
%We introduced DP-MicroAdam, a differentially private extension of MicroAdam that addresses the challenges of adaptive optimization under privacy constraints. 

We introduced DP-MicroAdam, a differentially private adaptive optimizer that integrates gradient sparsity and error feedback with standard DP mechanisms. Our theoretical analysis establishes convergence guarantees for stochastic non-convex optimization, showing that DP-MicroAdam achieves the optimal $\mathcal{O}(1/\sqrt{T})$ rate, up to privacy-dependent constants. Empirically, across vision tasks including CIFAR-10 training, ImageNet training, and CIFAR-10/100 fine-tuning, DP-MicroAdam consistently outperforms existing adaptive baselines and attains competitive or superior accuracy compared to DP-SGD while requiring minimal hyperparameter tuning.
Our hypothesis is that DP-MicroAdam, by using only the top-$K$ coordinates, mitigates the bias from which simple DP-Adam suffers. Future work could confirm the optimizer’s superiority experimentally on more diverse tasks, and also derive a theoretical analysis of the clipping impact and the bias to explain why DP-MicroAdam is more robust than DP-Adam.

\section*{Acknowledgments}
We thank Ionu\c{t}-Vlad Modoranu for sharing with us his implementation of MicroAdam and for his valuable feedback and insights throughout the development of this work. We are also grateful to Christoph Lampert for his helpful comments on earlier versions of the manuscript.
Kalinin’s and Cyffers' research was funded by the Austrian Science Fund (FWF) 10.55776/COE12. Hudi\c{s}teanu's research was funded by ISTernship Fellowship. 

\newpage
\bibliographystyle{plainnat}  
\bibliography{biblio2}

\newpage
\appendix
\section*{Appendix}

\section{Theoretical convergence result}
\label{app:theo}

\subsection{Detailed Assumptions}

\begin{assumption} \label{as:1}
The gradient compressor $\mathcal T : \mathbb{R}^d \to \mathbb{R}^d$ is $q$-contractive with $0 \leq q < 1$, i.e.,
\[
\|\mathcal{T}(x) - x\| \leq q \|x\|, \quad \text{for any } x \in \mathbb{R}^d.
\]
\end{assumption}

The compression we use is the TopK compressor $T_k$, which selects the top $k$ coordinates in absolute value. This is known to be contractive with $q = \sqrt{1 - \nicefrac{k}{d}}$. Another popular contractive compressor is the optimal low-rank projection of gradient shaped as a $d\times d$ matrix, in which case $q=\sqrt{1-R/d}$ where $R$ is the projection rank. 

\begin{assumption} \label{as:2}
The error compressor $Q : \mathbb{R}^d \to \mathbb{R}^d$ is unbiased and $\omega$-bounded with $\omega \geq 0$, namely,
\[
\mathbb{E}[Q(x)] = x, \quad \|Q(x) - x\| \leq \omega \|x\|, \quad \text{for any } x \in \mathbb{R}^d.
\]
\end{assumption}

The randomized rounding used in Algorithm~\ref{alg:dp-microadam} satisfies these properties. For a vector $x\in\R^d$ with $\delta = \min_i x_i$ and $\Delta = \max_i x_i$, let $\hat{x}_i \eqdef \lfloor \frac{x_i-\delta}{u} + \xi \rfloor u + \delta$ be the $i$-th coordinate of the quantized vector $\hat{x}$, where $\xi\sim\textrm{U}[0,1]$ is the uniform random variable and $u=\frac{\Delta-\delta}{2^b-1}$ is the quantization level. Then
    \begin{equation*}
        \E[\hat x] = x, \quad \|\hat x - x\| \le \frac{\sqrt{d-2}}{2^b-1}\frac{\Delta-\delta}{\sqrt{\Delta^2+\delta^2}}\|x\|, \quad \text{for all } x\in\R^d.
    \end{equation*}

\begin{assumption}[Lower bound and smoothness] \label{as:3}
The loss function \(f : \mathbb{R}^d \to \mathbb{R}\) is lower bounded by some \(f^* \in \mathbb{R}\) and \(L\)-smooth, i.e.,
\[
\|\nabla f(\theta) - \nabla f(\theta')\| \leq L \|\theta - \theta'\|, \quad \text{for any } \theta, \theta' \in \mathbb{R}^d.
\]
\end{assumption}

\begin{assumption}[Unbiased and bounded stochastic gradient] \label{as:4}
For all iterates \(t \geq 1\), the stochastic gradient \(g_t\) is unbiased and uniformly bounded by a constant \(G \geq 0\), i.e.,
\[
\mathbb{E}[g_t] = \nabla f(\theta_t), \quad \|g_t\| \leq G.
\]
\end{assumption}

\begin{assumption}[Bounded variance] \label{as:5}
For all iterates \(t \geq 1\), the variance of the stochastic gradient \(g_t\) is uniformly bounded by some constant \(\sigma_g^2 \geq 0\), i.e.,
\[
\mathbb{E}\left[\|g_t - \nabla f(\theta_t)\|^2\right] \leq \sigma_g^2.
\]
\end{assumption}

\subsection{Convergence proof}
To prove Theorem~\ref{TheoremNonconvexO}, we rewrite DP-MicroAdam as a variant of AMSGrad optimizer in Algorithm~\ref{algo:theory} to show convergence.

\begin{algorithm}[H]
    \caption{\label{algo:theory} Noisy-MicroAdam: Analytical View}
    \begin{algorithmic}[1]
        \State Input: parameters $\beta_1,\,\beta_2\in(0,1)$, $\epsilon_s>0$, step-size $\eta>0$, $\theta_1\in\R^d$, $e_{1}=m_0=v_0=\hat v_0=0_d$
        \For{$t=\{1, 2, ..., T\}$}
            \State{$g_{t} = \widetilde\nabla_{\theta} f(\theta_t)$ \hfill$\diamond$ Compute unbiased stochastic gradient}
            \State{$\tilde g_{t}=\mathcal T(g_{t}+e_{t} +  \textcolor{myblue}{\eta_t})$ \hfill$\diamond$ Add accumulated error $e_t$ and compress\label{line:topk1} }
            \State{$e_{t+1}=\mathcal{Q} (e_{t}+g_{t}-\tilde g_{t})$ \hfill$\diamond$ Update and compress the error}
            \State{$m_t=\beta_1 m_{t-1}+(1-\beta_1)\tilde g_t$ \hfill$\diamond$ Update first-order gradient moment}
            \State{$v_t=\beta_2 v_{t-1}+(1-\beta_2)\tilde g_t^2$ \hfill$\diamond$ Update second-order gradient moment}
            \State{$\hat v_t=\max(v_t,\hat v_{t-1})$ \hfill$\diamond$ Apply AMSGrad normalization\label{line:v}}
            \State{$\theta_{t+1}=\theta_{t}-\eta\frac{\change{m_t}}{\sqrt{\hat v_t+\epsilon_s}}$ \hfill$\diamond$ Update the model parameters}
        \EndFor
    \end{algorithmic}
\end{algorithm}

our objective with these two compressors, $\mathcal{T}$ and $\mathcal{Q}$, is to approximate the dense gradient information $g_t+e_t$ using two compressed vectors: $\tilde{g}_t = \mathcal{C}(g_t+e_t+\textcolor{myblue}{\eta_t})$ and $\mathcal{Q}(g_t+e_t+\textcolor{myblue}{\eta_t} - \tilde{g}_t)$. However, in doing so, we inevitably lose some information about $g_t+e_t+\textcolor{myblue}{\eta_t}$ depending on the degree of compression applied to each term. Thus, the  condition $(1+\omega)q<1$ required by our analysis can be seen as preventing excessive loss of information due to compression.

 At time step $t$, let the uncompressed stochastic gradient be $g_t = \widetilde\nabla_\theta f(\theta_t)$, the error accumulator be $e_t$, and the compressed gradient after the error correction be $\tilde g_{t}=\mathcal C(g_{t}+e_{t}+\textcolor{myblue}{\eta_t})$. The second moment computed by the compressed gradients is denoted as $v_t=\beta_2 v_{t-1}+(1-\beta_2) \tilde g_t^2$, and $\hat v_t=\max\{\hat v_{t-1}, v_t\}$ is the AMSGrad normalization for the second-order momentum. Besides the first-order gradient momentum $m_t$ used in the algorithm description, we define similar running average sequence $m'_t$ based on the uncompressed gradients $g_t$.
	\begin{align*}
		m_t=\beta_1 m_{t-1}+(1-\beta_1)\tilde g_t \quad & \textrm{and} \quad m_t'=\beta_1 m_{t-1}'+(1-\beta_1) g_t,
	\end{align*}
	Note that $m_t'$ is used only in the analysis, we do not need to store or compute it. By construction we have
    \begin{equation*}
    m_t=(1-\beta_1)\sum_{\tau=1}^t \beta_1^{t-\tau} \tilde g_\tau, \quad
    m_t'=(1-\beta_1)\sum_{\tau=1}^t \beta_1^{t-\tau} g_\tau
    \end{equation*}

	Denote by $\zeta_t = e_{t+1} - (e_t + g_t +\textcolor{myblue}{\eta_t} - \tilde g_t) = \mathcal Q(e_t + g_t+\textcolor{myblue}{\eta_t} - \tilde g_t) - (e_t + g_t +\textcolor{myblue}{\eta_t} - \tilde g_t)$ the compression noise from $\mathcal Q$. Due to unbiasedness of the compressor $\mathcal Q$ (see Assumption \ref{as:2}), we have $\E[\zeta_t \mid \theta_t, g_t, \tilde g_t, e_t] = 0$. Also, from the update rule of $e_{t+1}$ we get $e_{t+1} = e_t + g_t +\textcolor{myblue}{\eta_t} - \tilde g_t + \zeta_t$. Moreover, we use the following auxiliary sequences,
	\begin{align*}
        & \mathcal E_{t+1}\eqdef \beta_1\mathcal E_t + (1-\beta_1)e_{t+1} = (1-\beta_1)\sum_{\tau=1}^{t+1} \beta_1^{t+1-\tau} e_\tau. \\
	   & \mathcal Z_{t+1}\eqdef \beta_1\mathcal Z_t + (1-\beta_1)\zeta_{t+1} =  (1-\beta_1)\sum_{\tau=1}^{t+1} \beta_1^{t+1-\tau} \zeta_\tau \\
        & \textcolor{myblue}{\mathcal H_{t+1} \eqdef \beta_1 \mathcal H_t + (1-\beta_1)\eta_{t+1} = (1-\beta_1) \sum_{\tau=1}^{t+1} \beta_1^{t+1-\tau} \eta_\tau}.
	\end{align*}
    
Before proving the convergence itself, we need few lemmas:

\begin{Lemma} \label{lemma:m_t,m_t'}
	Under Assumptions~\ref{as:1}-\ref{as:5}, for all iterates $t$ and $T$ we have
	\begin{equation*}
		\|m_t'\|\leq G, \quad\text{and}\quad \sum_{t=1}^T\mathbb E[\|m_t'\|^2]\leq T\sigma_g^2 + \sum_{t=1}^T \mathbb E[\|\nabla f(\theta_t)\|^2].
	\end{equation*}
\end{Lemma}

\begin{proof}
The first part follows from triangle inequality and the Assumption~\ref{as:4} on bounded stochastic gradient:
\begin{equation*}
    \|m_t'\| = (1-\beta_1)\left\|\sum_{\tau=1}^t \beta_1^{t-\tau} g_{\tau} \right\|\leq (1-\beta_1)\sum_{\tau=1}^t \beta_1^{t-\tau} \|g_{\tau}\| \le G.
\end{equation*}
For the second claim, the expected squared norm of average stochastic gradient can be bounded by
\begin{equation}\label{var-decomp}
    \E\left[\|g_t\|^2\right]
    = \E\left[\|g_t-\nabla f(\theta_t))\|^2\right] + \E[\|\nabla f(\theta_t)\|^2]\\
    \le \sigma_g^2 + \E[\|\nabla f(\theta_t)\|^2],
\end{equation}
where we use Assumption~\ref{as:4} that $g_t$ is unbiased with bounded variance. Let $g_{t,j}$ denote the $j$-th coordinate of $g_t$. Applying Jensen's inequality for the squared norm, we get
\begin{eqnarray*}
    \E[\|m_t'\|^2]
    &=& \E\left[\left\|(1-\beta_1)\sum_{\tau=1}^t\beta_1^{t-\tau} g_\tau\right\|^2\right]\\
    &\leq& (1-\beta_1)\sum_{\tau=1}^t \beta_1^{t-\tau}\E[\|g_\tau\|^2]\\
    &\leq& \sigma_g^2 + (1-\beta_1)\sum_{\tau=1}^t \beta_1^{t-\tau}\mathbb E[\|\nabla f(\theta_{\tau})\|^2],
\end{eqnarray*}
Summing over $t=1,\dots,T$, we obtain
\begin{align*}
    \sum_{t=1}^T\E[\|m_t'\|^2]
    \le T\sigma_g^2 + (1-\beta_1)\sum_{t=1}^T\sum_{\tau=1}^t \beta_1^{t-\tau}\mathbb E[\|\nabla f(\theta_{\tau})\|^2]
    \leq T\sigma_g^2+\sum_{t=1}^T \mathbb E[\|\nabla f(\theta_t)\|^2],
\end{align*}
which completes the proof.
\end{proof}

\begin{Lemma} \label{lemma:bound e_t}
	Let $\q=(1+\omega)q<1$. Under Assumptions~\ref{as:1}-\ref{as:5}, for all iterates $t$ we have
	\begin{align*}
		& \textcolor{myblue}{\E[}\|e_{t}\|^2  \textcolor{myblue}{]}\leq \frac{4\q^2}{(1-\q^2)^2} \textcolor{myblue}{(}G  \textcolor{myblue}{+\sqrt{d}\sigma)^2},\\
		&\E[\|e_{t+1}\|^2]\leq \frac{ 4\q^2}{(1-\q^2)^2} \textcolor{myblue}{(}\sigma_g^2  \textcolor{myblue}{+d\sigma^2)} + \frac{2\q^2}{1-\q^2}\sum_{\tau=1}^t \left(\frac{1+\q^2}{2}\right)^{t-\tau} \E[\|\nabla f(\theta_\tau)\|^2],\\
            & \textcolor{myblue}{(\E[\|e_{t}\|^4 
        ])^\frac{1}{4}\leq \frac{2\q}{(1-\q^2)} (G  +\sqrt{d}\sigma)}.
	\end{align*}
\end{Lemma}

\begin{proof}
We start by using Assumption~\ref{as:1}, \ref{as:2} on compression, Young's inequality\textcolor{myblue}{, and Minkowski inequality} to get
\begin{align}
     \textcolor{myblue}{\E[}\|e_{t+1}\|^2 \textcolor{myblue}{]}
    &= \textcolor{myblue}{\E[}\|\mathcal Q(g_{t}+e_{t} +  \textcolor{myblue}{\eta_t}-\mathcal C(g_{t}+e_{t}+ \textcolor{myblue}{\eta_t}))\|^2 \textcolor{myblue}{]} \nonumber\\
    &\leq (1+\omega)^2q^2 \textcolor{myblue}{\E[}\|g_{t}+e_{t}+ \textcolor{myblue}{\eta_t}\|^2 \textcolor{myblue}{]} \nonumber\\
    &\leq \q^2(1+\rho) \textcolor{myblue}{\E[}\|e_{t}\|^2  \textcolor{myblue}{]}+ \q^2\left(1+\frac{1}{\rho}\right) \textcolor{myblue}{\E[}\|g_{t} + \textcolor{myblue}{\eta_t}\|^2  \textcolor{myblue}{]} \nonumber\\
    &\leq \frac{1+\q^2}{2}\textcolor{myblue}{\E[}\|e_{t}\|^2\textcolor{myblue}{]} + \frac{2\q^2}{1-\q^2}\textcolor{myblue}{\E[}\|g_{t} + \textcolor{myblue}{\eta_t}\|^2\textcolor{myblue}{]} \label{eq:e_t 0}\\ 
    &\leq \frac{1+\q^2}{2}\textcolor{myblue}{\E[}\|e_{t}\|^2\textcolor{myblue}{]} + \frac{2\q^2}{1-\q^2}\textcolor{myblue}{(\sqrt{\E[\|g_{t}\|^2]} + \sqrt{\E[\|\eta_t\|^2]})^2} \nonumber\\
    &\leq \frac{1+\q^2}{2}\textcolor{myblue}{\E[}\|e_{t}\|^2\textcolor{myblue}{]} + \frac{2\q^2}{1-\q^2}\textcolor{myblue}{(G +\sqrt{d}\sigma)^2} \nonumber, 
\end{align}
where \eqref{eq:e_t 0} is derived by choosing $\rho=\frac{1-\q^2}{2\q^2}$ and the fact that $\q<1$. For the first claim we recursively apply the obtained inequality and use bounded gradient Assumption \ref{as:4}. For the second claim, initialization $e_1=0$, \textcolor{myblue}{the fact that $g_\tau$ and $\eta_\tau$ are independent for all interates $\tau \geq 1$,} and the obtained recursion imply
\begin{eqnarray*}
    \E[\|e_{t+1}\|^2]
    &\leq& \frac{2\q^2}{1-\q^2} \sum_{\tau=1}^t \left(\frac{1+\q^2}{2}\right)^{t-\tau} \E[\|g_{\tau} \textcolor{myblue}{+\eta_{\tau}}\|^2]  \\
    &=& \frac{2\q^2}{1-\q^2} \sum_{\tau=1}^t \left(\frac{1+\q^2}{2}\right)^{t-\tau} \textcolor{myblue}{(\E[\|g_{\tau}\|^2] + \E[\|\eta_{\tau}\|^2])}  \\
    &\overset{\eqref{var-decomp}}{\leq}& \frac{4\q^2}{(1-\q^2)^2}\textcolor{myblue}{(}\sigma_g^2 \textcolor{myblue}{+ d\sigma^2)}+ \frac{2\q^2}{1-\q^2}\sum_{\tau=1}^t \left(\frac{1+\q^2}{2}\right)^{t-\tau} \E[\|\nabla f(\theta_{\tau})\|^2]. \nonumber
\end{eqnarray*}

\textcolor{myblue}{Now, for the third claim, we have that
\begin{eqnarray*}
     \E[\|e_{t+1}\|^4 ] ^\frac{1}{4}
    &=& \q \E[\|g_{t}+e_{t}+\eta_t\|^4]^\frac{1}{4} \\
    &\leq& \q (\E[\|g_{t}\|^4]^\frac{1}{4}+\E[\|e_{t}\|^4]^\frac{1}{4}+\E[\|\eta_{t}\|^4]^\frac{1}{4}) \\
    &\leq& \q \E[\|e_{t}\|^4]^\frac{1}{4}+ \q(G + \sqrt{d} \sigma).
\end{eqnarray*}
Unrolling the recursion and using the fact that \(\frac{\q}{1-\q} \leq \frac{2\q}{1-\q^2}\) we get the above upper bound.}
\end{proof}

\begin{Lemma} \label{lemma:bound zeta_t}
Let $\q=(1+\omega)q<1$. Under Assumptions~\ref{as:1}-\ref{as:5}, for all iterates $t$ we have
\begin{align*}
    \textcolor{myblue}{\E[}\|\zeta_t\| \textcolor{myblue}{^2]}\leq \omega\textcolor{myblue}{^2} q\textcolor{myblue}{^2} \left(1 + \frac{2\q}{1-\q^2} \right)\textcolor{myblue}{^2}  \textcolor{myblue}{(}G \textcolor{myblue}{+ \sqrt{d}\sigma)^2} \quad\text{and}\quad \textcolor{myblue}{\E[}\|\zeta_t\| \textcolor{myblue}{^4]^\frac{1}{4}}\leq \omega q \left(1 + \frac{2\q}{1-\q^2} \right) \textcolor{myblue}{(}G \textcolor{myblue}{+\sqrt{d}\sigma)}.
\end{align*}
\end{Lemma}
\begin{proof}
Using the bounds defining compressors and Lemma \ref{lemma:bound e_t}, we get
\begin{align*}
     \textcolor{myblue}{\E[}\|\zeta_t\| \textcolor{myblue}{^2]}
    &= \textcolor{myblue}{\E[}\|\mathcal Q(e_t+g_t \textcolor{myblue}{+\eta_t}-\tilde g_t) - (e_t+g_t\textcolor{myblue}{+ \eta_t}-\tilde g_t)\|\textcolor{myblue}{^2]} \nonumber\\
    &\le \omega\textcolor{myblue}{^2} \textcolor{myblue}{\E[}\| e_t+g_t\textcolor{myblue}{+\eta_t}-\tilde g_t\|\textcolor{myblue}{^2]}
    = \omega \textcolor{myblue}{^2} \textcolor{myblue}{\E[}\| e_t+g_t\textcolor{myblue}{+\eta_t}- \mathcal C(e_t+g_t\textcolor{myblue}{+\eta_t})\|\textcolor{myblue}{^2]} \nonumber\\
    &\leq \omega\textcolor{myblue}{^2} q\textcolor{myblue}{^2} \textcolor{myblue}{\E[}\|e_t+g_t+\textcolor{myblue}{\eta_t}\| \textcolor{myblue}{^2]} \nonumber\\
    &\leq \textcolor{myblue}{\omega^2q^2 \left( \sqrt{\E[\|e_t\|^2]} + \sqrt{\E[\|g_t\|^2]} + \sqrt{\E[\|\eta_t\|^2]}\right)^2} \nonumber \\
    &\leq \textcolor{myblue}{\omega^2 q^2 \left(1+\frac{2\q}{1-\q^2} \right)^2 (G + \sqrt{d}\sigma)^2}
\end{align*}
For the second claim, using Minkowski inequality, we get:
\textcolor{myblue}{
    \begin{align*}
     \E[\|\zeta_t\| ^4]^{\frac{1}{4}}
    &\leq \omega q\E[\|e_t+g_t+\eta_t\| ^4]^{\frac{1}{4}} \nonumber\\
    &\leq \omega q \left( \E[\|e_t\|^4]^{\frac{1}{4}} + \E[\|g_t\|^4]^{\frac{1}{4}} + \E[\|\eta_t\|^4]^{\frac{1}{4}}\right)\nonumber \\
    &\leq \omega q \left(1+\frac{2\q}{1-\q^2} \right) (G + \sqrt{d}\sigma).
    \end{align*}}
\end{proof}
\begin{Lemma} \label{lemma:bound big E_t}
	For the moving average error sequence $\mathcal E_t$, it holds that
	\begin{align*}
		\sum_{t=1}^T \E[\|\mathcal E_t\|^2]\leq \frac{4T\q^2}{(1-\q^2)^2}\textcolor{myblue}{(}\sigma_g^2 \textcolor{myblue}{+ d\sigma^2)} + \frac{4\q^2}{(1-\q^2)^2} \sum_{t=1}^T \E[\|\nabla f(\theta_t)\|^2 ].
	\end{align*}
\end{Lemma}

\begin{proof}
	Let $e_{t,j}$ be the $j$-th coordinate of $e_{t}$ and denote
    $$K_t \eqdef \sum_{\tau=1}^t \left(\tfrac{1+\q^2}{2}\right)^{t-\tau} \E[\|\nabla f(\theta_\tau)\|^2].$$
    Applying Jensen's inequality and Lemma \ref{lemma:bound e_t}, we get
	\begin{align*}
		\E[\|\mathcal E_t\|^2]
        &=\E\left[\left\|(1-\beta_1)\sum_{\tau=1}^t\beta_1^{t-\tau} e_\tau\right\|^2\right]\\
		&\le (1-\beta_1)\sum_{\tau=1}^t \beta_1^{t-\tau}\mathbb E[\|e_\tau\|^2]\\
		&\le \frac{4\q^2}{(1-\q^2)^2}\textcolor{myblue}{(}\sigma_g^2 \textcolor{myblue}{+ d\sigma^2)}+\frac{2\q^2(1-\beta_1)}{(1-\q^2)}\sum_{\tau=1}^t \beta_1^{t-\tau} K_{\tau},
	\end{align*}
	Summing over $t=1,\dots,T$ and using the technique of geometric series summation leads to
	\begin{align*}
		\sum_{t=1}^T \E[\|\mathcal E_t\|^2]
        &\leq \frac{ 4T\q^2}{(1-\q^2)^2}\textcolor{myblue}{(}\sigma_g^2 \textcolor{myblue}{+ d\sigma^2)} + \frac{ 2\q^2(1-\beta_1)}{(1-\q^2)}\sum_{t=1}^T \sum_{\tau=1}^t \beta_1^{t-\tau} K_{\tau}\\
        &\leq \frac{ 4T\q^2}{(1-\q^2)^2}\textcolor{myblue}{(}\sigma_g^2 \textcolor{myblue}{+ d\sigma^2)} + \frac{ 2\q^2}{(1-\q^2)}\sum_{t=1}^T K_{t}\\
		&= \frac{ 4T\q^2}{(1-\q^2)^2}\textcolor{myblue}{(}\sigma_g^2 \textcolor{myblue}{+ d\sigma^2)} + \frac{ 2\q^2}{(1-\q^2)}\sum_{t=1}^T\sum_{\tau=1}^t \left(\frac{1+\q^2}{2}\right)^{t-\tau} \E[\|\nabla f(\theta_\tau)\|^2]\\
		&\leq \frac{ 4T\q^2}{(1-\q^2)^2}\textcolor{myblue}{(}\sigma_g^2 \textcolor{myblue}{+ d\sigma^2)} + \frac{ 4\q^2}{(1-\q^2)^2} \sum_{t=1}^T \E[\|\nabla f(\theta_t)\|^2],
	\end{align*}
 The desired result is obtained.
\end{proof}

\begin{Lemma} \label{lemma:bound v_t}
	Let $\q=(1+\omega)q<1$. Under Assumptions~\ref{as:1}-\ref{as:5}, for all iterates $t\in [T]$ and coordinates $i\in [d]$, the following bound holds
    $$\textcolor{myblue}{\E[}\hat v_{t,i}\textcolor{myblue}{]}\leq \frac{4(1+\q^2)^3}{(1-\q^2)^2}\textcolor{myblue}{(}G \textcolor{myblue}{+ \sqrt{d}\sigma)^2}.$$
\end{Lemma}

\begin{proof}
Lemma~\ref{lemma:bound e_t} and Assumption~\ref{as:4} imply
\begin{align*}
    \textcolor{myblue}{\E[}\|\tilde g_t\|^2\textcolor{myblue}{]}&=\textcolor{myblue}{\E[}\|\mathcal C(g_t+e_t\textcolor{myblue}{+\eta_t})\|^2\textcolor{myblue}{]}\\
    &\leq \textcolor{myblue}{\E[}\|\mathcal C(g_t+e_t\textcolor{myblue}{+\eta_t})-(g_t+e_t\textcolor{myblue}{+\eta_t})+(g_t+e_t\textcolor{myblue}{+\eta_t})\|^2\textcolor{myblue}{]}\\
    &\leq 2(q^2+1)\textcolor{myblue}{\E[}\|g_t+e_t \textcolor{myblue}{+\eta_t}\|^2\textcolor{myblue}{]}\\
    &\leq 2(q^2+1)\textcolor{myblue}{\left(\sqrt{\E[\|g_t\|^2]} + \sqrt{\E[\|e_t\|^2]} + \sqrt{\E[\|\eta_t\|^2]}\right)^2}\\
    &\leq 2(q^2+1)\textcolor{myblue}{\left(1+\frac{2\q}{1-\q^2}\right)^2(G+\sqrt{d}\sigma)^2} \\
    &\leq 2(q^2+1)\textcolor{myblue}{\left(\frac{1+\q^2}{1-\q^2}\right)^2(G+\sqrt{d}\sigma)^2} \\
    &\leq \textcolor{myblue}{\frac{2(1+\q^2)^3}{(1-\q^2)^2}(G+\sqrt{d}\sigma)^2}.
\end{align*}
It's then easy to show by the updating rule of $\hat v_t$, there exists a $j\in[t]$ such that $\hat v_{t,i}=v_{j,i}$. Then
\begin{align*}
    \textcolor{myblue}{\E[}\hat v_{t,i}\textcolor{myblue}{]}=(1-\beta_2)\sum_{\tau=1}^j \beta_2^{j-\tau} \textcolor{myblue}{\E[} \tilde g_{\tau,i}^2 \textcolor{myblue}{]}\leq \textcolor{myblue}{\frac{2(1+\q^2)^3}{(1-\q^2)^2}(G+\sqrt{d}\sigma)^2} \leq \textcolor{myblue}{\frac{4(1+\q^2)^3}{(1-\q^2)^2}(G+\sqrt{d}\sigma)^2},
\end{align*}
which concludes the claim.
\end{proof}

\begin{Lemma}  \label{lemma:bound difference}
	For $D_t\eqdef \frac{1}{\sqrt{\hat v_{t-1}+\epsilon_s}}-\frac{1}{\sqrt{\hat v_t+\epsilon_s}}$ we have
	\begin{align*}
		& \sum_{t=1}^T  \textcolor{myblue}{\E[}\|D_t\|^2  \textcolor{myblue}{]} \leq \frac{d}{\epsilon_s},\quad  \textcolor{myblue}{\sum_{t=1}^T \sqrt{\E[\|D_t\|^2]} \leq \sqrt{T}\frac{\sqrt{d}}{\sqrt{\epsilon_s}}}.
	\end{align*}
\end{Lemma}

\begin{proof}
	By the update rule, we have $\hat v_{t-1,i}\leq \hat v_{t,i}$ for any iterate $t$ and coordinate $i\in[d]$. Note the fact that for $a\geq b>0$, it holds that
	\begin{equation*}
		(a-b)^2\leq (a-b)(a+b)=a^2-b^2.
	\end{equation*}
	Therefore,
	\begin{equation*}
		\sum_{t=1}^T \textcolor{myblue}{\E[}\|D_t\|^2 \textcolor{myblue}{]}
        =\sum_{t=1}^T \sum_{i=1}^d \textcolor{myblue}{\E\left[\left(\frac{1}{\sqrt{\hat v_{t-1,i}+\epsilon_s}}-\frac{1}{\sqrt{\hat v_{t,i}+\epsilon_s}}\right)^2 \right]}\\
		\leq \sum_{t=1}^T \sum_{i=1}^d \textcolor{myblue}{\E \left[ \left(\frac{1}{\hat v_{t-1,i}+\epsilon_s}-\frac{1}{\hat v_{t,i}+\epsilon_s}\right) \right]}\\
		\leq \frac{d}{\epsilon_s},.
	\end{equation*}

        \textcolor{myblue}{For the second inequality, after applying Cauchy-Schwarz, we have
        \begin{equation*}
		\sum_{t=1}^T \sqrt{\E[\|D_t\|^2]} \leq \sqrt{T}\sqrt{\sum_{t=1}^T\E[\|D_t\|^2 ]} \leq \sqrt{T}\frac{\sqrt{d}}{\sqrt{\epsilon_s}},
	\end{equation*}
	which gives the desired result.}
\end{proof}

The lemmas allow us to move to the proof of the full Theorem~\ref{TheoremNonconvexO}

\begin{proof}
Similar to the proof of \textrm{Comp-AMS}, we define two virtual iterates $\theta'_t$ and $x_t$.
\begin{align*}
    \theta_{t+1}' &\eqdef\theta_{t+1}-\eta \frac{\mathcal E_{t+1}}{\sqrt{\hat v_t+\epsilon_s}} \\
    x_{t+1} &\eqdef\theta_{t+1}'-\eta \frac{\beta_1}{1-\beta_1} \frac{m_t' + \mathcal Z_t \textcolor{myblue}{+ \mathcal H_t}}{\sqrt{\hat v_t+\epsilon_s}}.
\end{align*}

Then, we derive the recurrence relation for each sequence as follows:
\begin{align*}
    \theta_{t+1}'
    &= \theta_{t+1}-\eta \frac{\mathcal E_{t+1}}{\sqrt{\hat v_t+\epsilon_s}} \\
    &=\theta_t-\eta\frac{(1-\beta_1)\sum_{\tau=1}^{t} \beta_1^{t-\tau}\tilde g_\tau+(1-\beta_1)\sum_{\tau=1}^{t+1} \beta_1^{t+1-\tau} e_\tau}{\sqrt{\hat v_t+\epsilon_s}}\\
    &=\theta_t-\eta\frac{(1-\beta_1)\sum_{\tau=1}^{t} \beta_1^{t-\tau}(\tilde g_\tau+ e_{\tau+1})+(1-\beta)\beta_1^t e_1}{\sqrt{\hat v_t+\epsilon_s}}\\
    &=\theta_t-\eta\frac{(1-\beta_1)\sum_{\tau=1}^{t} \beta_1^{t-\tau}(g_\tau+ e_{\tau} + \zeta_{\tau}  \textcolor{myblue}{+\eta_{\tau}}) }{\sqrt{\hat v_t+\epsilon_s}}\\
    &=\theta_t-\eta\frac{(1-\beta_1)\sum_{\tau=1}^{t} \beta_1^{t-\tau} e_\tau}{\sqrt{\hat v_t+\epsilon_s}} - \eta\frac{m_t'}{\sqrt{\hat v_t+\epsilon_s}} - \eta\frac{\mathcal Z_t}{\sqrt{\hat v_t+\epsilon_s}} \textcolor{myblue}{- \eta\frac{\mathcal H_t}{\sqrt{\hat v_t+\epsilon_s}}} \\
    &=\theta_t-\eta\frac{\mathcal E_t}{\sqrt{\hat v_{t-1}+\epsilon_s}}-\eta\frac{m_t'}{\sqrt{\hat v_t+\epsilon_s}}+\eta\left(\frac{1}{\sqrt{\hat v_{t-1}+\epsilon_s}}-\frac{1}{\sqrt{\hat v_t+\epsilon_s}}\right)\mathcal E_t - \eta\frac{\mathcal Z_t}{\sqrt{\hat v_t+\epsilon_s}} \textcolor{myblue}{- \eta\frac{\mathcal H_t}{\sqrt{\hat v_t+\epsilon_s}}}\\
    &=\theta_t'-\eta\frac{m_t'}{\sqrt{\hat v_t+\epsilon_s}}+\eta\left(\frac{1}{\sqrt{\hat v_{t-1}+\epsilon_s}}-\frac{1}{\sqrt{\hat v_t+\epsilon_s}}\right)\mathcal E_t - \eta\frac{\mathcal Z_t}{\sqrt{\hat v_t+\epsilon_s}} \textcolor{myblue}{- \eta\frac{\mathcal H_t}{\sqrt{\hat v_t+\epsilon_s}}}\\
    &= \theta_t'-\eta \frac{m_t' + \mathcal Z_t \textcolor{myblue}{+ \mathcal H_t}}{\sqrt{\hat v_t+\epsilon_s}}+\eta D_t\mathcal E_t,
\end{align*}
where we used the fact that $\tilde g_{t}+e_{{t+1}}=g_{t}+e_{t} + \zeta_t \textcolor{myblue}{+ \eta_t}$ with quantization noise $\zeta_t$, and $e_{0}=0$ at initialization. Next, for the $x_t$ iterates we have
\begin{align*}
    x_{t+1}
    &=\theta_{t+1}'-\eta\frac{\beta_1}{1-\beta_1} \frac{m_t' + \mathcal Z_t \textcolor{myblue}{+ \mathcal H_t} }{\sqrt{\hat v_t+\epsilon_s}} \\
    &=\theta_t' - \eta\frac{m_t' + \mathcal Z_t \textcolor{myblue}{+ \mathcal H_t}}{\sqrt{\hat v_t+\epsilon_s}}-\eta\frac{\beta_1}{1-\beta_1} \frac{m_t' + \mathcal Z_t \textcolor{myblue}{+ \mathcal H_t}}{\sqrt{\hat v_t+\epsilon_s}}+\eta D_t\mathcal E_t \\
    &=\theta_t'-\eta \frac{\beta_1 (m_{t-1}' + \mathcal Z_{t-1} \textcolor{myblue}{+ \mathcal H_{t-1}})+(1-\beta_1) (g_t+\zeta_t \textcolor{myblue}{+\eta_t})+\frac{\beta_1^2}{1-\beta_1}(m_{t-1}'+\mathcal Z_{t-1 } \textcolor{myblue}{+ \mathcal H_{t-1}})+\beta_1 (g_t+\zeta_t \textcolor{myblue}{+\eta_t})}{\sqrt{\hat v_t+\epsilon_s}}\\
    &\quad+\eta D_t\mathcal E_t \\
    &=\theta_t'-\eta\frac{\beta_1}{1-\beta_1}\frac{m_{t-1}'+\mathcal Z_{t-1} \textcolor{myblue}{+ \mathcal H_{t-1}}}{\sqrt{\hat v_t+\epsilon_s}}-\eta\frac{g_t+\zeta_t \textcolor{myblue}{+\eta_t}}{\sqrt{\hat v_t+\epsilon_s}}+\eta D_t\mathcal E_t \\
    &=x_t-\eta\frac{g_t + \zeta_t  \textcolor{myblue}{+\eta_t}}{\sqrt{\hat v_t+\epsilon_s}}+\eta\frac{\beta_1}{1-\beta_1} D_t( m_{t-1}' + \mathcal Z_{t-1} \textcolor{myblue}{+ \mathcal H_{t-1}})+\eta D_t\mathcal E_t.
\end{align*}

Next we apply smoothness of the loss function $f$ over the iterates $x_t$. From the gradient Lipschitzness we have
\begin{align*}
    f(x_{t+1})\leq f(x_t)+\langle \nabla f(x_t), x_{t+1}-x_t\rangle+\frac{L}{2}\| x_{t+1}-x_t\|^2.
\end{align*}
Due to unbiasedness of the compressor $\mathcal Q$ (see Assumption \ref{as:2}), we have $\E[\zeta_t | g_t, \tilde g_t, e_t, \hat v_t] = 0$. Since $\eta_t$ is Gaussian with mean zero, it follows that  $\E[\eta_t] = 0$. Taking expectation, we obtain
\begin{eqnarray}
    \E[f(x_{t+1})]-\E[f(x_t)]
    &\leq& -\eta\mathbb E\left[\left\langle \nabla f(x_t), \frac{g_t + \zeta_t \textcolor{myblue}{+\eta_t}}{\sqrt{\hat v_t+\epsilon_s}}\right\rangle\right] \nonumber\\
    &&+\eta \E\left[\left\langle \nabla f(x_t), \frac{\beta_1}{1-\beta_1}D_t(m_{t-1}' + \mathcal Z_{t-1} \textcolor{myblue}{+ \mathcal H_{t-1}})+D_t\mathcal E_t\right\rangle\right] \nonumber\\
    &&+\frac{\eta^2L}{2} \E\left[\left\|\frac{g_t + \zeta_t \textcolor{myblue}{+\eta_t}}{\sqrt{\hat v_t+\epsilon_s}}-\frac{\beta_1}{1-\beta_1}D_t(m_{t-1}' + \mathcal Z_{t-1} \textcolor{myblue}{+ \mathcal H_{t-1}})- D_t\mathcal E_t\right\|^2\right] \nonumber\\
    &=&\underbrace{-\eta \E\left[\left\langle \nabla f(\theta_t), \frac{g_t}{\sqrt{\hat v_t+\epsilon_s}}\right\rangle\right]}_{I}\\
    &&+\underbrace{\eta \E\left[\left\langle \nabla f(x_t), \frac{\beta_1}{1-\beta_1}D_t(m_{t-1}'+\mathcal Z_{t-1} \textcolor{myblue}{+ \mathcal H_{t-1}})+D_t\mathcal E_t\right\rangle\right]}_{II} \nonumber\\
    &&+ \underbrace{\frac{\eta^2L}{2} \E\left[\left\|\frac{g_t + \zeta_t  \textcolor{myblue}{+\eta_t}}{\sqrt{\hat v_t+\epsilon_s}}-\frac{\beta_1}{1-\beta_1}D_t(m_{t-1}'+\mathcal Z_{t-1} \textcolor{myblue}{+ \mathcal H_{t-1}})- D_t\mathcal E_t\right\|^2\right]}_{III} \nonumber\\
    &&+ \underbrace{\eta\E\left[\left\langle \nabla f(\theta_t)-\nabla f(x_t), \frac{g_t}{\sqrt{\hat v_t+\epsilon_s}} \right\rangle\right]}_{IV}, \label{eq0}
\end{eqnarray}

In the following, we bound all the four terms highlighted above.

\textbf{Bounding term I.} We have
\begin{eqnarray}
    I&=&
     -\eta\E\left[\left\langle \nabla f(\theta_t), \frac{g_t}{\sqrt{\hat v_{t-1}+\epsilon_s}}\right\rangle\right]+\eta\E\left[\left\langle \nabla f(\theta_t), \left(\frac{1}{\sqrt{\hat v_{t-1}+\epsilon_s}}-\frac{1}{\sqrt{\hat v_t+\epsilon_s}}\right) g_t\right\rangle\right] \nonumber\\
    &\leq& -\eta\E\left[\left\langle \nabla f(\theta_t), \frac{g_t}{\sqrt{\hat v_{t-1}+\epsilon_s}}\right\rangle\right]+\eta\textcolor{myblue}{\E\left[\|\nabla f(\theta_t)\| \| D_t g_t\|\right]} \nonumber\\
    &\leq& -\eta\E\left[\left\langle \nabla f(\theta_t), \frac{g_t}{\sqrt{\hat v_{t-1}+\epsilon_s}}\right\rangle\right]+\eta\textcolor{myblue}{ G \sqrt{\E\left[\| D_t\|^2 \right]} \sqrt{\E\left[\| g_t\|^2 \right]}} \nonumber\\
    &\leq& -\eta\E\left[\left\langle \nabla f(\theta_t), \frac{\nabla f(\theta_t)}{\sqrt{\hat v_{t-1}+\epsilon_s}}\right\rangle\right]+\eta G^2 \textcolor{myblue}{\sqrt{\E\left[\| D_t\|^2 \right]}}  \nonumber\\
    &\textcolor{myblue}{\overset{(a)}{\leq}}& -\frac{\eta}{\sqrt{\frac{4(1+\q^2)^3}{(1-\q^2)^2}\textcolor{myblue}{(}G \textcolor{myblue}{+\sqrt{d}\sigma)^2}+\epsilon_s}}\mathbb E[\|\nabla f(\theta_t)\|^2]+\eta G^2\textcolor{myblue}{\sqrt{\E\left[\| D_t\|^2 \right]}}, \label{eq:I}
\end{eqnarray}
where we use Assumption~\ref{as:4}, Lemma~\ref{lemma:bound v_t}, the Cauchy-Schwarz inequality, \textcolor{myblue}{and (a) is derived as follows:
\begin{align*}
-\eta\E\left[\left\langle \nabla f(\theta_t), \frac{\nabla f(\theta_t)}{\sqrt{\hat v_{t-1}+\epsilon_s}}\right\rangle\right] 
&= -\eta\E\left[ \sum_{i=1}^{d} \nabla f(\theta_t)_i \cdot \frac{\nabla f(\theta_t)_i}{\sqrt{\hat v_{t-1,i}+\epsilon_s}} \right] \\
&= -\eta \sum_{i=1}^{d} \E \left[ \nabla f(\theta_t)_i^2 \cdot \frac{1}{\sqrt{\hat v_{t-1,i}+\epsilon_s}} \right] \\
&\leq -\eta \sum_{i=1}^{d} \E \left[ \nabla f(\theta_t)_i^2 \right] \cdot \E \left[\frac{1}{\sqrt{\hat v_{t-1,i}+\epsilon_s}} \right] \\
&\overset{(b)}{\leq} -\eta \sum_{i=1}^{d} \E \left[ \nabla f(\theta_t)_i^2 \right] \cdot \frac{1}{\sqrt{\E[\hat v_{t-1,i}+\epsilon_s}]} \\
&\leq -\frac{\eta}{\sqrt{\frac{4(1+\q^2)^3}{(1-\q^2)^2}\textcolor{myblue}{(}G \textcolor{myblue}{+\sqrt{d}\sigma)^2}+\epsilon_s}} \cdot \sum_{i=1}^{d} \E \left[ \nabla f(\theta_t)_i^2 \right] \\
&= -\frac{\eta}{\sqrt{\frac{4(1+\q^2)^3}{(1-\q^2)^2}(G +\sqrt{d}\sigma)^2+\epsilon_s}}\mathbb E[\|\nabla f(\theta_t)\|^2],
\end{align*}
where (b) is due to Jensen's inequality applied to the convex function $\frac{1}{\sqrt{x}}.$
}

\textbf{Bounding term II.} By the definition of $\mathcal E_t$, $\mathcal Z_t$, and  $\textcolor{myblue}{ \mathcal H_t}$ and applying Jensen's inequality, we know that
\begin{align*}
\textcolor{myblue}{\E[}\|\mathcal E_t\| \textcolor{myblue}{^2]} &\leq (1-\beta_1)\sum_{\tau=1}^t \beta_1^{t-\tau}\textcolor{myblue}{\E[}\| e_\tau\| \textcolor{myblue}{^2]}\leq \left( \frac{2\q}{1-\q^2} \right)\textcolor{myblue}{^2} \textcolor{myblue}{(}G \textcolor{myblue}{+\sqrt{d}\sigma)^2},\\
\textcolor{myblue}{\E[}\|\mathcal Z_t\| \textcolor{myblue}{^2]} &\leq (1-\beta_1)\sum_{\tau=1}^t \beta_1^{t-\tau} \textcolor{myblue}{\E[}\|\zeta_\tau\| \textcolor{myblue}{^2]}\leq \omega\textcolor{myblue}{^2} q\textcolor{myblue}{^2} \left(1 + \frac{2\q}{1-\q^2} \right)\textcolor{myblue}{^2}  \textcolor{myblue}{(}G \textcolor{myblue}{+\sqrt{d}\sigma)^2}, \\
\textcolor{myblue}{\E[\| \mathcal H_t\|^2]} &\leq \textcolor{myblue}{(1-\beta_1) \sum_{\tau=1}^t \beta_1^{t-\tau} \E[\| \eta_\tau\| ^2] \leq d \sigma^2}.
\end{align*}

\textcolor{myblue}{Similarly, we derive that:
\begin{align*}
\E[\|\mathcal E_t\|^4]^\frac{1}{4} &\leq \left( \frac{2\q}{1-\q^2} \right) (G +\sqrt{d}\sigma),\\
\E[\|\mathcal Z_t\|^4]^\frac{1}{4} &\leq \omega q\left( 1 + \frac{2\q}{1-\q^2} \right) (G +\sqrt{d}\sigma),\\
\E[\|\mathcal H_t\|^4]^\frac{1}{4} &\leq \sqrt{d}\sigma.
\end{align*}
}

Then, because of the smoothness of $f(\theta)$, we have
\begin{align}
    II & \leq\eta\E\left[\left\langle  \nabla f(\theta_t),\frac{\beta_1}{1-\beta_1}D_t(m_{t-1}'+\mathcal Z_{t-1} \textcolor{myblue}{+ \mathcal H_{t-1}})+D_t\mathcal E_t\right\rangle\right] \nonumber\\
    &\quad + \eta\E\left[\left\langle \nabla f(x_t)-\nabla f(\theta_t),\frac{\beta_1}{1-\beta_1}D_t(m_{t-1}'+\mathcal Z_{t-1} \textcolor{myblue}{+ \mathcal H_{t-1}})+D_t\mathcal E_t\right\rangle\right] \nonumber\\
    &\leq \eta \underbrace{\E\left[\|\nabla f(\theta_t)\| \left\|\frac{\beta_1}{1-\beta_1}D_t(m_{t-1}'+\mathcal Z_{t-1} \textcolor{myblue}{+ \mathcal H_{t-1}})+D_t\mathcal E_t \right\|\right]}_{\tilde I} \nonumber \\
    &\quad +\eta^2 L  \underbrace{\E\left[\left\|\frac{\frac{\beta_1}{1-\beta_1}m_{t-1}'+\frac{\beta_1}{1-\beta_1}\mathcal Z_{t-1} \textcolor{myblue}{+ \frac{\beta_1}{1-\beta_1}\mathcal H_{t-1}}+\mathcal E_t}{\sqrt{\hat v_{t-1}+\epsilon_s}}\right\| \left\|\frac{\beta_1}{1-\beta_1}D_t(m_{t-1}' + \mathcal Z_{t-1} \textcolor{myblue}{+ \mathcal H_{t-1}})+D_t\mathcal E_t\right\|\right]}_{\tilde{II}} \nonumber
    %&\leq \eta C_1 G^2 \mathbb E[\|D_t\|_1]+\frac{\eta^2 C_1^2 LG^2}{\sqrt\epsilon_s}\mathbb E[\|D_t\|_1],  \label{eq:II}
\end{align}

\textcolor{myblue}{
Now we bound $\tilde I$ and $\tilde{II}$ separately:
\begin{align*}
    \tilde I & \leq G \E\left[ \left\|\frac{\beta_1}{1-\beta_1}D_t(m_{t-1}'+\mathcal Z_{t-1} + \mathcal H_{t-1})+D_t\mathcal E_t \right\|\right] \\
    &\leq G \E\left[ \left\|D_t \right\| \left\| \frac{\beta_1}{1-\beta_1}(m_{t-1}'+\mathcal Z_{t-1} + \mathcal H_{t-1})+ \mathcal E_t \right\|\right] \\
    &\leq G \sqrt{\E[\|D_t\|^2]}   \sqrt{\E\left[ \left\| \frac{\beta_1}{1-\beta_1}(m_{t-1}'+\mathcal Z_{t-1} + \mathcal H_{t-1})+ \mathcal E_t \right\|^2\right]} \\
    &\leq G \sqrt{\E[\|D_t\|^2]}   \sqrt{\E\left[ \left\| \frac{\beta_1}{1-\beta_1}(m_{t-1}'+\mathcal Z_{t-1} + \mathcal H_{t-1})+ \mathcal E_t \right\|^2\right]} \\
   &\leq G \sqrt{\E[\|D_t\|^2]}   \left(  \frac{\beta_1}{1-\beta_1} \left( \sqrt{\E \left[ \left\|m_{t-1}' \right\|^2 \right]} +  \sqrt{\E \left[ \left\|\mathcal Z_{t-1} \right\|^2 \right]} + \sqrt{\E \left[ \left\|\mathcal H_{t-1} \right\|^2 \right]}\right) + \sqrt{\E \left[ \left\|\mathcal E_t  \right\|^2 \right]} \right) \\
   &\leq G \sqrt{\E[\|D_t\|^2]}  \left( \frac{\beta_1}{1-\beta_1}\left( 1+\omega q \left(1 + \frac{2\q}{1-\q^2} \right)\right)+\frac{2\q}{1-\q^2} \right) (G + \sqrt{d} \sigma)  \\
    &\leq C_1 (G + \sqrt{d} \sigma)^2 \sqrt{\E[\|D_t\|^2]}, 
\end{align*}}
where $C_1= \frac{\beta_1}{1-\beta_1}\left( 1+\omega q \left(1 + \frac{2\q}{1-\q^2} \right)\right)+\frac{2\q}{1-\q^2}$.
\textcolor{myblue}{
\begin{align*}
    \tilde{II} &\leq \frac{1}{\sqrt\epsilon_s} \E\left[\left\|\frac{\beta_1}{1-\beta_1}m_{t-1}'+\frac{\beta_1}{1-\beta_1}\mathcal Z_{t-1} + \frac{\beta_1}{1-\beta_1}\mathcal H_{t-1}+\mathcal E_t \right\| \left\|\frac{\beta_1}{1-\beta_1}D_t(m_{t-1}' + \mathcal Z_{t-1} + \mathcal H_{t-1})+D_t\mathcal E_t\right\|\right] \\
    &\leq \frac{1}{\sqrt\epsilon_s} \E\left[\left\|\frac{\beta_1}{1-\beta_1}m_{t-1}'+\frac{\beta_1}{1-\beta_1}\mathcal Z_{t-1} + \frac{\beta_1}{1-\beta_1}\mathcal H_{t-1}+\mathcal E_t \right\|^2 \left\| D_t\right\|\right] \\
    &\leq \frac{1}{\sqrt\epsilon_s} \sqrt{\E[\|D_t\|^2]} \sqrt{\E\left[\left\|\frac{\beta_1}{1-\beta_1}m_{t-1}'+\frac{\beta_1}{1-\beta_1}\mathcal Z_{t-1} + \frac{\beta_1}{1-\beta_1}\mathcal H_{t-1}+\mathcal E_t \right\|^4\right]} \\
    &\leq \frac{1}{\sqrt\epsilon_s} \sqrt{\E[\|D_t\|^2]} \left( \frac{\beta_1}{1-\beta_1} \left( \E \left[ \left\|m_{t-1}' \right\|^4 \right]^\frac{1}{4} +  \E \left[ \left\|\mathcal Z_{t-1} \right\|^4 \right]^\frac{1}{4} + \E \left[ \left\|\mathcal H_{t-1} \right\|^4 \right]]^\frac{1}{4} \right) + \E \left[ \left\|\mathcal E_t  \right\|^4 \right]]^\frac{1}{4} \right)^2 \\
    &\leq \frac{C_1^2 (G+\sqrt{d}\sigma)^2}{\sqrt\epsilon_s} \sqrt{\E[\|D_t\|^2]}.
\end{align*}
So, the final bound for $II$ is:
\begin{align}
    II \leq \eta C_1 (G + \sqrt{d} \sigma)^2 \sqrt{\E[\|D_t\|^2]} + \frac{\eta^2 C_1^2 (G+\sqrt{d}\sigma)^2}{\sqrt\epsilon_s} \sqrt{\E[\|D_t\|^2]} \label{eq:II}
\end{align}
}

\textbf{Bounding term III.} This term can be bounded as follows:
\begin{align}
    III & \leq \eta^2 L\E\left[\left\|\frac{g_t+\zeta_t \textcolor{myblue}{+\eta_t}}{\sqrt{\hat v_t+\epsilon_s}}\right\|^2\right] + \eta^2 L\E\left[\left\|\frac{\beta_1}{1-\beta_1}D_t(m_{t-1}'+\mathcal Z_{t-1} \textcolor{myblue}{+ \mathcal H_{t-1}})- D_t\mathcal E_t\right\|^2\right] \nonumber\\
    &\textcolor{myblue}{\overset{(a)}{\leq}}\frac{\textcolor{myblue}{3}\eta^2 L}{\epsilon_s}\E[\| g_t-\nabla f(\theta_t)+\nabla f(\theta_t)\|^2]+\frac{\textcolor{myblue}{3}\eta^2 L}{\epsilon_s}\E[\|\zeta_t\|^2] +\textcolor{myblue}{\frac{3\eta^2 L}{\epsilon_s}\E[\|\eta_t\|^2]} \nonumber \\
    &\quad+\eta^2 L\E\left[\left\|D_t\left(\frac{\beta_1}{1-\beta_1}(m_{t-1}'+ \mathcal Z_{t-1} \textcolor{myblue}{+ \mathcal H_{t-1}})  + \mathcal  E_t\right)\right\|^2\right] \nonumber\\
    &\textcolor{myblue}{\overset{(b)}{\leq}} \frac{\textcolor{myblue}{3}\eta^2 L}{\epsilon_s}\E[\| g_t-\nabla f(\theta_t)+\nabla f(\theta_t)\|^2]+\frac{\textcolor{myblue}{3}\eta^2 L}{\epsilon_s}\E[\|\zeta_t\|^2] +\textcolor{myblue}{\frac{3\eta^2 L}{\epsilon_s}\E[\|\eta_t\|^2]} \nonumber\\
    &\quad+\textcolor{myblue}{\frac{\eta^2}{\sqrt \epsilon_s} L\E\left[\left\|D_t\left(\frac{\beta_1}{1-\beta_1}(m_{t-1}'+ \mathcal Z_{t-1} + \mathcal H_{t-1})  + \mathcal  E_t\right)\right\| \left\| \frac{\beta_1}{1-\beta_1}(m_{t-1}'+\mathcal Z_{t-1} + \mathcal H_{t-1})  + \mathcal  E_t\right\| \right]} \nonumber\\
    &\leq \frac{\textcolor{myblue}{3}\eta^2 L}{\epsilon_s}\E[\|\nabla f(\theta_t)\|^2] + \frac{\textcolor{myblue}{3}\eta^2 L \textcolor{myblue}{(}\sigma_g^2 \textcolor{myblue}{+d\sigma^2)}}{\epsilon_s} + \frac{\textcolor{myblue}{3}\eta^2 L}{\epsilon_s}\omega^2 q^2 \left(1 + \frac{2q}{1-q^2} \right)^2 \textcolor{myblue}{(}G \textcolor{myblue}{+\sqrt{d}\sigma)^2} \nonumber\\
    &\quad +\frac{\eta^2 C_1^2 L \textcolor{myblue}{(}G \textcolor{myblue}{+\sqrt{d}\sigma)^2} }{\textcolor{myblue}{\sqrt\epsilon_s}}\sqrt{\E[\|D_t\|^2]} \nonumber\\
    &\leq \frac{\textcolor{myblue}{3}\eta^2 L}{\epsilon_s}\E[\|\nabla f(\theta_t)\|^2] + \frac{\textcolor{myblue}{3}\eta^2 L (\textcolor{myblue}{(}\sigma_g^2 \textcolor{myblue}{+d\sigma^2)} + C_2^2\textcolor{myblue}{(}G \textcolor{myblue}{+\sqrt{d}\sigma)^2})}{\epsilon_s} \nonumber\\
    &\quad+\frac{\eta^2 C_1^2 L\textcolor{myblue}{(}G \textcolor{myblue}{+\sqrt{d}\sigma)^2})}{\textcolor{myblue}{\sqrt\epsilon_s}} \sqrt{\E[\|D_t\|^2]},  \label{eq:III}
\end{align}
where $C_2 = \omega q (1 + \frac{2q}{1-q^2})$. \textcolor{myblue}{We used Cauchy-Schwarz inequality for (a) and the fact that $||D_t|| \leq \frac{1}{\sqrt\epsilon_s}$ for (b).}

\textbf{Bounding term IV.} We have
\begin{align}
    IV
    &= \eta\E\left[\left\langle \nabla f(\theta_t)-\nabla f(x_t), \frac{g_t}{\sqrt{\hat v_{t-1}+\epsilon_s}} \right\rangle\right] %\nonumber\\
    + \eta\E\left[\left\langle \nabla f(\theta_t)-\nabla f(x_t), \left(\frac{1}{\sqrt{\hat v_t+\epsilon_s}}-\frac{1}{\sqrt{\hat v_{t-1}+\epsilon_s}}\right) g_t \right\rangle\right] \nonumber\\
    &\leq \eta\E\left[\left\langle \nabla f(\theta_t)-\nabla f(x_t), \frac{\nabla f(\theta_t)}{\sqrt{\hat v_{t-1}+\epsilon_s}} \right\rangle\right] %\nonumber\\
    +\eta^2 L\mathbb E\left[\left\|\frac{\frac{\beta_1}{1-\beta_1}(m_{t-1}'+\mathcal Z_{t-1} \textcolor{myblue}{+\mathcal H_{t-1}})+\mathcal E_t}{\sqrt{\hat v_{t-1}+\epsilon_s}}\right\| \|D_t g_t\|\right] \nonumber\\
    &\overset{(a)}{\leq} \frac{\eta \rho}{2\epsilon_s}\mathbb E[\|\nabla f(\theta_t)\|^2]+\frac{\eta}{2\rho}\mathbb E[\|\nabla f(\theta_t)-\nabla f(x_t)\|^2]+\frac{\eta^2 C_1LG^2}{\sqrt\epsilon_s}  \sqrt{\E[\|D_t\|^2]}  \nonumber\\
    &\overset{(b)}{\leq} \frac{\eta \rho}{2\epsilon_s}\mathbb E[\|\nabla f(\theta_t)\|^2]+\frac{\eta^3 L^{\change{2}}}{2\rho}\mathbb E\left[\left\|\frac{\frac{\beta_1}{1-\beta_1}m_{t-1}'+\frac{\beta_1}{1-\beta_1}\mathcal Z_{t-1} \textcolor{myblue}{+\frac{\beta_1}{1-\beta_1}\mathcal H_{t-1}}+\mathcal E_t}{\sqrt{\hat v_{t-1}+\epsilon_s}}\right\|^2\right]+\frac{\eta^2 C_1LG^2}{\sqrt\epsilon_s} \sqrt{\E[\|D_t\|^2]} \nonumber \\
    &\leq \frac{\eta \rho}{2\epsilon_s}\mathbb E[\|\nabla f(\theta_t)\|^2]+\frac{\eta^3 L^{\change{2}}}{2\rho}\frac{C_1^2\textcolor{myblue}{(}G \textcolor{myblue}{+\sqrt{d}\sigma)^2}}{\epsilon_s}+\frac{\eta^2L C_1\textcolor{myblue}{(}G \textcolor{myblue}{+\sqrt{d}\sigma)^2}}{\sqrt\epsilon_s} \sqrt{\E[\|D_t\|^2]},  \label{eq:IV}
\end{align}
where (a) is due to Young's inequality and (b) is based on Assumption~\ref{as:3}. Now integrating \eqref{eq:I}, \eqref{eq:II}, \eqref{eq:III}, \eqref{eq:IV} into \eqref{eq0},
\squeezevspace
\begin{eqnarray*}
    I &\leq& -\frac{\eta}{C_0}\mathbb E[\|\nabla f(\theta_t)\|^2]+\eta \textcolor{myblue}{(}G \textcolor{myblue}{+\sqrt{d}\sigma)^2} \sqrt{\E[\|D_t\|^2]} \\
    II &\le& \eta C_1 \textcolor{myblue}{(}G \textcolor{myblue}{+\sqrt{d}\sigma)^2} \sqrt{\E[\|D_t\|^2]}+\frac{\eta^2 C_1^2 L\textcolor{myblue}{(}G \textcolor{myblue}{+\sqrt{d}\sigma)^2}}{\sqrt\epsilon_s}\sqrt{\E[\|D_t\|^2]} \\
    III &\le& \frac{\eta^2 L}{\epsilon_s}\mathbb E[\|\nabla f(\theta_t)\|^2] + \frac{\eta^2 L(\sigma^2 + C_2^2 \textcolor{myblue}{(}G \textcolor{myblue}{+\sqrt{d}\sigma)^2})}{\epsilon_s}+ \frac{\eta^2 C_1^2 L\textcolor{myblue}{(}G \textcolor{myblue}{+\sqrt{d}\sigma)^2}}{\textcolor{myblue}{\sqrt\epsilon_s}} \sqrt{\E[\|D_t\|^2]} \\
    IV &\le& \frac{\eta \rho}{2\epsilon_s}\mathbb E[\|\nabla f(\theta_t)\|^2]+\frac{\eta^3 L^{\change{2}}}{2\rho}\frac{C_1^2\textcolor{myblue}{(}G \textcolor{myblue}{+\sqrt{d}\sigma)^2}}{\epsilon_s}+\frac{\eta^2L C_1\textcolor{myblue}{(}G \textcolor{myblue}{+\sqrt{d}\sigma)^2}}{\sqrt\epsilon_s} \sqrt{\E[\|D_t\|^2]},
\end{eqnarray*}
and taking the telescoping summation over $t=1,\dots,T$, we obtain
\begin{align*}
    &\E[f(x_{T+1})-f(x_1)] \\
    &\leq \left( -\frac{\eta}{C_0}+\frac{\eta^2 L}{\epsilon_s}+\frac{\eta \rho}{2\epsilon_s}\right)\sum_{t=1}^T\E[\|\nabla f(\theta_t)\|^2]
    +\frac{T\eta^2 L (\textcolor{myblue}{(}\sigma_g^2 \textcolor{myblue}{+d\sigma^2)} + C_2^2\textcolor{myblue}{(}G \textcolor{myblue}{+\sqrt{d}\sigma)^2})}{\epsilon_s} \\
    &\quad + \frac{T\eta^3 L^{\change{2}} C_1^2\textcolor{myblue}{(}G \textcolor{myblue}{+\sqrt{d}\sigma)^2}}{2\rho\epsilon_s}
    + \textcolor{myblue}{\eta(G + \sqrt{d}\sigma)^2 \left[ (1+C_1) + \frac{\eta C_1 L}{\sqrt\epsilon_s}(2C_1+1)\right] \sum_{t=1}^T \sqrt{\E[\|D_t\|^2]}}
\end{align*}

Setting $\eta\leq \frac{\epsilon_s}{4L C_0}$ and choosing $\rho=\frac{\epsilon_s}{2C_0}$, we further arrive at
\begin{align*}
    \E[f(x_{T+1})-f(x_1)]
    &\leq -\frac{\eta}{2C_0}\sum_{t=1}^T\E[\|\nabla f(\theta_t)\|^2]
    +\frac{T\eta^2 L (\textcolor{myblue}{(}\sigma_g^2 \textcolor{myblue}{+d\sigma^2)} + C_2^2\textcolor{myblue}{(}G \textcolor{myblue}{+\sqrt{d}\sigma)^2})}{\epsilon_s}\\
    &\quad +\frac{T\eta^3 L^{\change{2}} C_0C_1^2 \textcolor{myblue}{(}G \textcolor{myblue}{+\sqrt{d}\sigma)^2}
}{\epsilon_s^2}
    + \frac{\eta (1+C_1)\textcolor{myblue}{(}G \textcolor{myblue}{+\sqrt{d}\sigma)^2 \sqrt{d}}}{\sqrt\epsilon_s}\textcolor{myblue}{\sqrt{T}} \\
    &\quad+ \frac{\eta^2 (1+2C_1)C_1L\textcolor{myblue}{(}G \textcolor{myblue}{+\sqrt{d}\sigma)^2\sqrt{d}}}{{\textcolor{myblue}{\sqrt\epsilon_s}}}\textcolor{myblue}{\sqrt{T}}.
\end{align*}
where the inequality follows from Lemma~\ref{lemma:bound difference}. Re-arranging terms, we get that
\begin{align*}
    &\frac{1}{T}\sum_{t=1}^T \E[\|\nabla f(\theta_t)\|^2] \\
    &\leq 2C_0\left(\frac{\E[f(x_1)-f(x_{T+1})]}{T\eta}
    +\frac{\eta L (\sigma_g^2 \textcolor{myblue}{+d\sigma^2} + C_2^2\textcolor{myblue}{(}G \textcolor{myblue}{+\sqrt{d}\sigma)^2})}{\epsilon_s}
    +\frac{\eta^2 L^{\change{2}} C_0C_1^2\textcolor{myblue}{(}G \textcolor{myblue}{+\sqrt{d}\sigma)^2}}{\epsilon_s^2}\right)\\
    &\quad+2C_0\left(\frac{(1+C_1)\textcolor{myblue}{(}G \textcolor{myblue}{+\sqrt{d}\sigma)^2 \sqrt{d}}}{T\sqrt\epsilon_s}\textcolor{myblue}{\sqrt{T}}
    +\frac{\eta (1+2C_1)C_1L\textcolor{myblue}{(}G \textcolor{myblue}{+\sqrt{d}\sigma)^2\sqrt{d}} }{T\textcolor{myblue}{\sqrt\epsilon_s}}\textcolor{myblue}{\sqrt{T}} \right)\\
    &\leq 2C_0\left(\frac{f(\theta_1)-f^*}{T\eta}
    +\frac{\eta L (\sigma_g^2 \textcolor{myblue}{+d\sigma^2}+ C_2^2\textcolor{myblue}{(}G \textcolor{myblue}{+\sqrt{d}\sigma)^2})}{\epsilon_s}
    +\frac {\eta^2 L^{\change{2}} C_0C_1^2\textcolor{myblue}{(}G \textcolor{myblue}{+\sqrt{d}\sigma)^2}}{\epsilon_s^2}\right)\\
    &\quad+2C_0\left(\frac{(1+C_1)\textcolor{myblue}{(}G \textcolor{myblue}{+\sqrt{d}\sigma)^2\sqrt{d}}}{\textcolor{myblue}{\sqrt{T}}\sqrt\epsilon_s}
    +\frac{\eta (1+2C_1)C_1L\textcolor{myblue}{(}G \textcolor{myblue}{+\sqrt{d}\sigma)^2\sqrt{d}}}{\textcolor{myblue}{\sqrt{T\epsilon_s}}} \right),
\end{align*}
where in the last inequality we used $x_1=\theta_1$ and the lower bound $f^* \le f(\theta)$. 

To get the rate mentioned in the main part, choose $\eta = \min\{\frac{\epsilon_s}{4LC_0}, \frac{1}{\sqrt{T}}\}$:
\begin{align*}
    &\frac{1}{T}\sum_{t=1}^T \E[\|\nabla f(\theta_t)\|^2] \\
    &\leq 2C_0\left( \max\left\{1, \frac{4LC_0}{\epsilon_s\sqrt{T}}\right\} \frac{f(\theta_1)-f^*}{\sqrt{T}}
    +\frac{L (\sigma_g^2 \textcolor{myblue}{+d\sigma^2} + C_2^2\textcolor{myblue}{(}G \textcolor{myblue}{+\sqrt{d}\sigma)^2})}{\epsilon_s\sqrt{T}}
    +\frac {L^2 C_0C_1^2\textcolor{myblue}{(}G \textcolor{myblue}{+\sqrt{d}\sigma)^2}}{\epsilon_s^2 T}\right)\\
    &\quad+2C_0\left(\frac{(1+C_1)\textcolor{myblue}{(}G \textcolor{myblue}{+\sqrt{d}\sigma)^2\sqrt{d}}}{\textcolor{myblue}{\sqrt{T}}\sqrt\epsilon_s}
    +\frac{(1+2C_1)C_1L\textcolor{myblue}{(}G \textcolor{myblue}{+\sqrt{d}\sigma)^2\sqrt{d}}}{\textcolor{myblue}{T \sqrt\epsilon_s}} \right) \\
    &\leq 2C_0\left(\frac{f(\theta_1)-f^*}{\sqrt{T}} + \frac{4LC_0}{\epsilon_s}\frac{f(\theta_1)-f^*}{T}
    +\frac{L (\sigma_g^2 \textcolor{myblue}{+d\sigma^2}+ C_2^2\textcolor{myblue}{(}G \textcolor{myblue}{+\sqrt{d}\sigma)^2})}{\epsilon_s\sqrt{T}} +  \frac{L^2 C_0C_1^2\textcolor{myblue}{(}G \textcolor{myblue}{+\sqrt{d}\sigma)^2}}{\epsilon_s^2 T} \right)\\
    &\quad+2C_0\left(\frac{(1+C_1)\textcolor{myblue}{(}G \textcolor{myblue}{+\sqrt{d}\sigma)^2\sqrt{d}}}{\textcolor{myblue}{\sqrt{T}}\sqrt\epsilon_s}
    +\frac{(1+2C_1)C_1L\textcolor{myblue}{(}G \textcolor{myblue}{+\sqrt{d}\sigma)^2\sqrt{d}}}{\textcolor{myblue}{T \sqrt\epsilon_s}} \right) \\
    &= 2C_0\left(\frac{f(\theta_1)-f^*}{\sqrt{T}}
    +\frac{L (\sigma_g^2 \textcolor{myblue}{+d\sigma^2}+ C_2^2\textcolor{myblue}{(}G \textcolor{myblue}{+\sqrt{d}\sigma)^2})}{\epsilon_s\sqrt{T}} + \frac{(1+C_1)\textcolor{myblue}{(}G \textcolor{myblue}{+\sqrt{d}\sigma)^2\sqrt{d}}}{\textcolor{myblue}{\sqrt{T}}\sqrt\epsilon_s} \right)
    + \mathcal{O}\left(\frac{\textcolor{myblue}{(}G \textcolor{myblue}{+\sqrt{d}\sigma)}^4}{T}\right).
\end{align*}
\end{proof}

\section{Model Architectures}
\label{appendix:architectures}

\paragraph{Wide-ResNet-16-4.}
Our CIFAR-10 and Imagenet experiments use a Wide-ResNet-16-4 architecture~\citep{zagoruyko2016wideresnet}. 
Wide-ResNets are a variant of residual networks in which the number of convolutional channels is increased by a widening factor $k$, while the total depth is reduced. We follow prior DP work~\citep{de2022unlocking, klause2022scalenorm} and set the depth to 16 layers and the widening factor to 4.

\paragraph{Normalization.}
To prevent privacy leakage through batch statistics, batch normalization is replaced with group normalization~\citep{wu2018groupnormalization}, using 16 groups per layer. Group normalization provides consistent feature scaling independent of the batch dimension, an essential property under small physical batches and Poisson subsampling.

\paragraph{Weight Standardization.}
We also apply weight standardization before each convolution, which normalizes convolutional filters to zero mean and unit variance. 

\paragraph{NF-ResNet.}
While our ImageNet experiments are conducted using the Wide-ResNet-16-4 architecture, we compare our results against the setup of \citet{de2022unlocking}, who trained an NF-ResNet-50 model under 
$(\varepsilon,\delta)=(8,8{\times}10^{-7})$-DP. 
The Normalized-Free Residual Network (NF-ResNet) replaces normalization layers with scaled activations and parameterizations, ensuring stable training without relying on batch statistics. 

\paragraph{DeiT Models.}
For private fine-tuning, we use Data-efficient Image Transformers (DeiT) pretrained on ImageNet. We evaluate DeiT-Tiny, DeiT-Small, and DeiT-Base, each based on the Vision Transformer (ViT) architecture with patch size $16\times16$ and input resolution $224\times224$. The classification head is replaced with a randomly initialized linear layer matching the target dataset (CIFAR-10 or CIFAR-100). 

\begin{figure*}[h]
\centering
\includegraphics[width=0.9\textwidth]{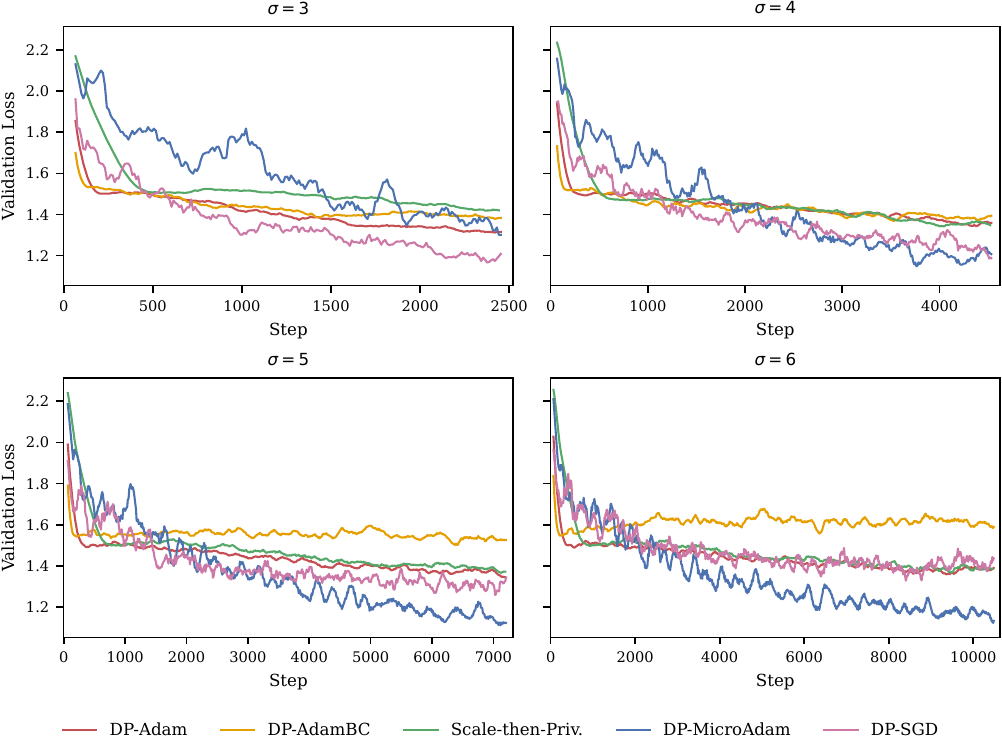}
\caption{Validation loss over training for CIFAR-10 under $(\varepsilon,\delta)=(8,10^{-5})$. Each subplot corresponds to a different Gaussian noise multiplier $\sigma$.}
\label{fig:val_loss_sigma}
\end{figure*}

\section{Additional Analysis and Results}

\subsection{CIFAR-10}
\label{app:moretables}
Unless otherwise stated, experiments use large effective batch sizes of approximately 4096 obtained via Poisson subsampling. 
We found that DP-MicroAdam is robust to the choice of clipping threshold and thus fix it to $C=1$ across all experiments. Since most related works also adopt $C=1$, this choice ensures fair comparability across methods. 
Results are reported as the median over three independent runs.

\subsection{Training on CIFAR-10}
We conduct experiments on the CIFAR-10 dataset.
The official training set of 50K images is split into 45K training and 5K validation examples, while final results are reported on the 10K test set. All images are normalized per channel and augmented with random crops and horizontal flips.

\paragraph{Model architecture.}
Following prior work on differentially private training~\citep{de2022unlocking,klause2022scalenorm,bu2024dpbitfit}, we adopt a Wide-ResNet-16-4~\citep{zagoruyko2016wideresnet}. 
Wide-ResNets use fewer layers but wider channels, which improves feature reuse, stabilizes gradient flow, and has been found empirically to be more robust to DP noise than deeper, narrower networks such as ResNet-20. 
To prevent privacy leakage from batch statistics, we replace batch normalization with group normalization~\citep{wu2018groupnormalization} using 16 groups. 
We also apply weight standardization before each convolution to further improve training stability.

\paragraph{Comparison setup.}
Using the training pipeline described above, we compare DP-MicroAdam to several established optimizers: DP-SGD~\citep{abadi2016}, DP-Adam, DP-AdamBC~\citep{tang2023dpadambc}, and Scale-then-Privatize~\citep{li2023delayedprecon}. 
All methods share the same architecture, data processing, and privacy accounting framework to ensure comparability. 
Standard hyperparameter settings are used for each optimizer. 
For adaptive methods, we adopt common defaults: $\epsilon = 10^{-8}$ for numerical stability, learning rate $\eta = 0.001$, and exponential decay rates $\beta_1 = 0.9$, $\beta_2 = 0.999$. 
For DP-SGD, we set the learning rate to $\eta = 4.0$. 
No weight decay is applied, and all optimizers use a fixed learning rate throughout training.

We fix the privacy budget to $(\varepsilon, \delta) = (8, 10^{-5})$ and conduct experiments with different Gaussian noise multipliers. For each noise level, training proceeds until the cumulative privacy loss reaches the target budget. 
\begin{table}[ht]
\centering
\caption{Test accuracy (\%) on CIFAR-10 under $(\varepsilon, \delta) = (8, 10^{-5})$ for different noise multipliers.}
\vspace{0.5em}
\label{tab:results_sigma}
\begin{tabular}{l|cccccc}
\toprule
\textbf{Noise} & \textbf{DP-Adam} & \textbf{DP-AdamBC} & \textbf{Scale-then-Priv.} & \textbf{DP-MicroAdam} & \textbf{DP-SGD} \\
\midrule
3 &64.4  &65.4  &61.2  &65.3  &\textbf{70.6}  \\
4 &65.1  &63.6  &64.5  &70.7  &\textbf{71.0}  \\
5 &65.5  &62.5  &64.2  &\textbf{71.4}  &68.4  \\
6 &65.0  &60.6  &63.6  &\textbf{70.1}  &67.0  \\
8 &63.5  &57.1  &63.5  &\textbf{69.4}  &65.9  \\
\bottomrule
\end{tabular}\end{table}

Table~\ref{tab:results_sigma} reports CIFAR-10 test accuracy for several differentially private optimizers under a fixed privacy budget of $(\varepsilon,\delta)=(8,10^{-5})$. 
DP-MicroAdam consistently outperforms adaptive baselines such as DP-Adam, DP-AdamBC, and Scale-then-Privatize, and achieves accuracy comparable to or exceeding DP-SGD at higher noise levels.
While traditional adaptive methods accumulate bias in their moment estimates and degrade rapidly as noise increases, DP-MicroAdam focuses updates on informative gradient coordinates and, as a result, maintains high accuracy even when Gaussian noise is large (e.g., $\sigma \ge 5$). 
 
The slightly lower performance observed at $\sigma=3$ is expected, as DP-MicroAdam updates only about 1\% of parameters per iteration and the number of steps allowed before reaching the privacy budget is relatively small. At higher noise levels, more steps can be taken, enabling the optimizer to reach its full potential.

We also compare all optimizers across different privacy budgets.
Table~\ref{tab:results_cifar} shows that DP-MicroAdam consistently outperforms other adaptive methods and slightly surpasses DP-SGD for all $\varepsilon > 1$, while its performance is less competitive only under the strictest privacy constraint.

\subsection{ImageNet}

In our experiments, we adopt EMA but do not use AugMult, as generating and processing multiple augmented views significantly increases the computational cost. 
Given the size and diversity of the ImageNet dataset, extensive augmentation may provide less benefit than on smaller datasets such as CIFAR-10. 
Following~\citet{de2022unlocking} we fix the Gaussian noise multiplier to $\sigma = 2.5$ and use a large effective batch size of 16{,}384. We report results for DP-MicroAdam using the same Wide-ResNet-16-4 architecture as in the CIFAR-10 experiments, while~\citet{de2022unlocking} employ a larger NF-ResNet-50 model. 

\subsection{Fine-tuning}

We run our experiments on a single GPU with \texttt{DeiT-Tiny}, \texttt{DeiT-Small}, and \texttt{DeiT-Base}, replacing the classification head with a randomly initialized linear layer corresponding to the target dataset.

All images are resized to $224\times224$, normalized using ImageNet statistics, and augmented with random crops and horizontal flips.  

We compare DP-MicroAdam against two leading approaches for sparse differentially private fine-tuning: 
DP-BiTFiT~\citep{bu2024dpbitfit}, which trains only the bias terms of each layer, and 
SPARTA~\citep{jang2025sparta}, which adaptively selects a small subset of parameters to update based on gradient magnitudes. 

\section{Validation Loss Curves Across Noise Levels}
To complement the accuracy results in Table~\ref{tab:results_sigma}, 
Figure~\ref{fig:val_loss_sigma} shows the validation loss trajectories for the Gaussian noise multipliers $\sigma \in \{3,4,5,6\}$. 
Each plot compares DP-MicroAdam with DP-Adam, DP-AdamBC, Scale-then-Priv., and DP-SGD under the same privacy budget $(\varepsilon,\delta)=(8,10^{-5})$. 
For each method, we report the median run over 5 independent experiments. 
To improve readability, loss curves are smoothed using a centered rolling mean over a window of 10 steps. This smoothing removes high-frequency noise and produces clearer visual trends without altering the relative behavior of the methods.

\section{Training Parameters}
This section details the hyperparameters used across all experimental settings considered in this paper. Tables~\ref{tab:cifar10_hparams}--\ref{tab:finetune_hparams} correspond respectively to the results reported in 
Tables~\ref{tab:results_sigma}–\ref{tab:results_finetune} of the main text.

\begin{table*}[h]
\centering
\caption{Hyperparameters for CIFAR-10 training with a WRN-16-4. %using Group Normalization and Weight Standardization. 
Left: adaptive optimizers (DP-Adam, DP-AdamBC, Scale-then-Priv., DP-MicroAdam). 
Right: DP-SGD.}
\vspace{0.4em}
\label{tab:cifar10_hparams}
\begin{tabular}{lcc}
\toprule
\textbf{Hyperparameter} & \textbf{Adaptive Optimizers} & \textbf{DP-SGD} \\
\midrule
Batch size & 4096 & 4096 \\
Learning rate $\eta$ & 0.001 &  4.0 \\
Decay rates $(\beta_1,\beta_2)$ & (0.9,\,0.999) & - \\
Clipping norm $C$ & 1.0 & 1.0 \\
Augmentation & RandomCrop, Flip & RandomCrop, Flip \\
\bottomrule
\end{tabular}
\end{table*}
%(32, padding=4)

\begin{table}[h]
\centering
\caption{Hyperparameters for ImageNet training under $(\varepsilon,\delta)=(8,8{\times}10^{-7})$.}
\vspace{0.4em}
\label{tab:imagenet_hparams}
\begin{tabular}{lcc}
\toprule
\textbf{Hyperparameter} & \textbf{DP-MicroAdam} & \textbf{DP-SGD} \\
\midrule
Model & WRN-16-4 (GN16, WS) & NF-ResNet-50 \\
Batch size & 16,384 & 16,384 \\
Learning rate $\eta$ & 0.001 &  4.0 \\
Decay rates $(\beta_1,\beta_2)$ & (0.9,\,0.999) & - \\
Clipping norm $C$ & 1.0 & 1.0 \\
Noise multiplier $\sigma$ & 2.5 & 2.5 \\
Augmentation & RandomCrop, Flip & AugMult ($K{=}4$) \\
EMA decay & 0.99999 & 0.99999 \\
\bottomrule
\end{tabular}
\end{table}

\begin{table}[h!]
\centering
\caption{Hyperparameters for fine-tuning on CIFAR-10 and CIFAR-100 under $(\varepsilon,\delta)=(8,10^{-5})$. 
%All models use patch size $16\times16$ and input resolution $224\times224$ (\texttt{patch16\_224}).
}
\vspace{0.4em}
\label{tab:finetune_hparams}
\begin{tabular}{lcc}
\toprule
\textbf{Hyperparameter} & \textbf{DP-MicroAdam} & \textbf{SPARTA / DP-BiTFiT} \\
\midrule
Batch size & 4096 & 500 \\
Learning rate $\eta$ & 0.001 & 0.001 \\
Scheduler & Cosine & Cosine \\
Warm-up ratio & 0.02 & 0.02 \\
Decay rates $(\beta_1,\beta_2)$ & (0.9,\,0.999) & – \\
Clipping norm $C$ & 1.0 & 1.0 \\
Noise multiplier $\sigma$ & 1.56 (CIFAR-10) / 2 (CIFAR-100) & 0.8\textcolor{gray}{*} \\ 
Epochs & 50 (CIFAR-10) / 90 (CIFAR-100) & 50 \\
Augmentation & RandomCrop, Flip & RandomCrop, Flip \\
\bottomrule
\end{tabular}
\par\vspace{0.4em}
\textit{ *Estimated from reported privacy budget, as not provided in the original paper.} 
\end{table}

\section{Runtime and Compute Configuration}
This section reports the approximate compute requirements for each experiment. Runtime measurements correspond to training until the privacy budget reaches $(\varepsilon,\delta)$.

\begin{table*}[h]
\centering
\caption{Approximate training time on a single H100 GPU to reach $(\varepsilon,\delta)=(8,10^{-5})$ on CIFAR-10 with WRN-16-4 under different noise multipliers $\sigma$. 
Training uses Poisson subsampling with an effective batch size of 4096 and a physical batch size of 128.}
\vspace{0.4em}
\label{tab:runtime_sigma}
\begin{tabular}{lccccc}
\toprule
\textbf{Noise multiplier $\sigma$} & 3 & 4 & 5 & 6 & 8 \\
\midrule
Training time & 1h 34m & 2h 47m & 4h 15m & 6h 00m & 10h 24m \\
Number of updates & 2480  & 4556  & 7227 &10492  &18798  \\
\bottomrule
\end{tabular}
\end{table*}

\begin{table*}[h]
\centering
\caption{Approximate training time for ImageNet under $(\varepsilon,\delta)=(8,8{\times}10^{-7})$. 
Both setups use an effective batch size of 16{,}384 and a physical batch size of 128.}
\vspace{0.4em}
\label{tab:imagenet_runtime}
\begin{tabular}{lcc}
\toprule
\textbf{Configuration} & \textbf{DP-MicroAdam} & \textbf{\citet{de2022unlocking}} \\
\midrule
Hardware & 8$\times$H100 GPUs & 32$\times$H100 GPUs \\
Training time & 8d 15h & 4d \\
Number of updates & 71{,}528 & 71{,}589 \\
\bottomrule
\end{tabular}
\end{table*}

\begin{table*}[h!]
\centering
\caption{Approximate fine-tuning time on CIFAR-10 and CIFAR-100 on a single H100 GPU to reach $(\varepsilon,\delta)=(8,10^{-5})$. 
Training uses an effective batch size of 4096 and a physical batch size of 128 for DeiT-Tiny and DeiT-Small, and 64 for DeiT-Base.}
\vspace{0.4em}
\label{tab:runtime_finetune}
\begin{tabular}{l|lccc}
\toprule
\textbf{Dataset} & \textbf{Model} & \textbf{Training time} & \textbf{Number of updates} \\
\midrule
\multirow{3}{*}{CIFAR-10} 
 & DeiT-Tiny  & 39m  & 539 \\
 & DeiT-Small & 1h 25m & 539 \\
 & DeiT-Base  & 4h 15m & 539 \\
\midrule
\multirow{3}{*}{CIFAR-100} 
 & DeiT-Tiny  & 1h 20m & 994 \\
 & DeiT-Small & 2h 37m & 994 \\
 & DeiT-Base  & 7h 33m & 994 \\
\bottomrule
\end{tabular}
\end{table*}

\end{document}